\documentclass[11pt,letterpaper]{article}

\usepackage[margin=1in]{geometry}
\usepackage[utf8]{inputenc}
\usepackage[T1]{fontenc}
\usepackage{lmodern}

\usepackage{amsmath, amssymb, amsfonts, mathtools}
\usepackage{amsthm}   
\usepackage{dsfont}
\usepackage{bbm}
\usepackage{comment}

\usepackage[maxfloats=100]{morefloats}
\usepackage{multirow}
\usepackage{wrapfig}
\usepackage{booktabs}
\usepackage{nicefrac}
\usepackage{microtype}
\usepackage{enumerate}
\usepackage{lipsum}
\usepackage{cuted}
\usepackage{float}
\usepackage[dvipsnames,table,xcdraw]{xcolor}
\usepackage{varioref}
\usepackage{rotating}
\usepackage{graphicx}
\usepackage{subcaption}
\usepackage[normalem]{ulem}
\usepackage{url}
\usepackage{algorithm}
\usepackage{algpseudocode}
\usepackage{hyperref}

\newtheorem{theorem}{Theorem}[section]

\newtheorem{lemma}[theorem]{Lemma}

\newtheorem{definition}[theorem]{Definition}
\newtheorem{assumption}[theorem]{Assumption}
\newtheorem{remark}[theorem]{Remark}

\newcommand{\fm}[1]{{\color{black}#1}}
\newcommand{\fr}[1]{{\color{black}#1}}

\newcommand{\fy}[1]{{\color{black}#1}}
\newcommand{\fyr}[1]{{\color{black}#1}}
\newcommand{\fyrr}[1]{{\color{black}#1}}
\usepackage[textsize=small]{todonotes}

\title{\LARGE \bf
Incentive-Aware Federated Averaging with Performance Guarantees under Strategic Participation
}

\author{Fateme Maleki$^{1}$, Krishnan Raghavan$^{2}$, and Farzad Yousefian$^{1}$%
\thanks{This work was supported in part by the U.S. Department of Energy under Awards DE-SC0025570 and DE-SC0023303.}%
\thanks{$^{1}$Maleki and Yousefian are with the Department of Industrial and Systems Engineering, Rutgers University, USA. \{fateme.maleki,farzad.yousefian\}@rutgers.edu.}%
\thanks{$^{2}$Raghavan is with Argonne National Laboratory, Lemont, IL, USA. kraghavan@anl.gov.}%
}

\begin{document}

\maketitle
\thispagestyle{empty}
\pagestyle{empty}

\begin{abstract}
Federated learning (FL) is a communication-efficient collaborative learning framework that enables model training across multiple agents with private local datasets. While the benefits of FL in improving global model performance are well established, individual agents may behave strategically, balancing the learning payoff against the cost of contributing their local data. Motivated by the need for FL frameworks that successfully retain participating agents, we propose an incentive-aware federated averaging method in which, at each communication round, clients transmit both their local model parameters and their updated training dataset sizes to the server. The dataset sizes are dynamically adjusted via a Nash equilibrium (NE)–seeking update rule that captures strategic data participation. \fyrr{Under a strongly monotone game setting, we} analyze the proposed method under convex and nonconvex global objective settings and establish performance guarantees for the resulting incentive-aware FL algorithm. \fyrr{Furthermore, under a merely monotone game setting, we consider a welfare loss minimization framework and establish the asymptotic convergence of the scheme.} Numerical experiments on the MNIST and CIFAR-10 datasets demonstrate that agents achieve competitive global model performance while converging to stable data participation strategies.
\end{abstract}


\section{Introduction}

Federated learning (FL) has recently emerged as a communication-efficient algorithmic framework that enables collaborative model training across multiple agents \cite{mcmahan2017communication,kairouz2021advances}. 
While this collaborative advantage has been rigorously studied~\cite{li2020fedprox,karimireddy2020scaffold,stich2019local}, as discussed by Blum et al. \cite{blum2021one}, the effectiveness of FL fundamentally depends on {\it the ability to recruit and retain a large number of participants} willing to contribute their data and resources. In existing FL frameworks, it is assumed that the data participation level for each agent remains constant throughout the learning task. However,
 a participating agent may achieve their training goal even when they unilaterally reduce their data contribution. Further, in settings where some participating agents are competitors, they might not have an incentive to consider full data participation. For these reasons, a fundamental challenge in federated learning arises around the following questions:  {(Q1)} How can we develop a game-theoretic framework for FL that captures agents’ learning payoffs and participation costs, while guaranteeing the existence of a Nash equilibrium (NE)?
 {(Q2)} Building on this formulation, how can we design an incentive-aligned FL algorithmic framework that enables strategic data participation?
 {(Q3)} Under such a scheme, can we establish simultaneous performance guarantees for both collaborative learning and the stability of the resulting noncooperative participation strategies?

To address (Q1), the work in~\cite{blum2021one} considers a constrained cost-minimization formulation where each agent minimizes its data contribution subject to a payoff lower bound. However, an NE may fail to exist in general. To overcome this issue,~\cite{murhekar2023incentives} studies an unconstrained net utility loss formulation—minimizing cost minus payoff—which is more natural from a game-theoretic perspective (cf.~\cite{myerson1981optimal,huber2001gaining}).

To address \fyrr{(Q2)}, recent works design incentive-aligned FL frameworks for strategic clients. Yi et al.~\cite{yi2025incentive} \fyrr{study participation incentives in FL under heterogeneous agent data qualities and characterize NE participation behavior. Alaei et al.~\cite{alaei2025incentivizing} use mechanism design to elicit truthful reporting and reward contributions. Further work includes contract theory for effort alignment \cite{karimireddy2022mechanisms}, penalties for truthfulness \cite{bornstein2024fact}, and Bayesian incentive compatibility under heterogeneity \cite{chakarov2024incentivizing} (see also~\cite{luo2023incentive,pandey2022fedtoken,tu2021incentive}).} 

\fyrr{In addressing (Q3), recent work~\cite{doshi2025incentivize} extends the framework in~\cite{murhekar2023incentives} and establishes convergence guarantees in terms of gradient norms, reaching the welfare-optimal Nash equilibrium under a budget-balanced monetary mechanism. Within this picture, two questions remain open. First, when the participation game admits a unique equilibrium, explicit communication-complexity bounds on the NE infeasibility itself, together with simultaneous bounds on the global FL objective, appear to be unavailable. Second, when the game admits multiple equilibria, the route to a welfare-improving equilibrium in the prior literature passes through monetary transfers, leaving open whether a welfare-selected equilibrium can be reached from the clients' intrinsic payoff and cost tradeoffs alone.}

{\bf Contributions.} 
(i) We introduce an incentive-aware federated averaging algorithm in which clients dynamically adjust their local dataset sizes via NE-seeking updates at each communication round. (ii) \fr{When the participation game is strongly monotone, we establish explicit communication-complexity bounds for the simultaneous convergence of the global FL objective and the \fyrr{NE infeasibility}, under both convex and nonconvex losses.} (iii) \fyrr{When the participation game is merely monotone and may admit multiple equilibria, we extend the framework to reach a welfare-selected equilibrium, through an iteratively regularized update, with explicit error bounds on the welfare gap and the NE infeasibility gap.} 

 \noindent {\bf Notation.} Vectors $x \in \mathbb{R}^n$ are treated as column vectors, with $x^\top$ representing the transpose. The symbol $\|\cdot\|$ denotes the Euclidean norm.
We define the Euclidean projection $\Pi_X[x]$ as the point in $X$ closest to $x$, i.e., $\Pi_X[x] \triangleq \arg\min_{y \in X} \|x - y\|$. The distance from a point $x$ to the set $X$ is given by $\text{dist}(x, X) \triangleq \|x - \Pi_X[x]\|$. We let $X^*$ denote the optimal solution set in~\eqref{prob:incentive_FL}. 

\section{Problem formulation}
\fr{
\subsection{Strongly monotone game}}
Consider a federated optimization problem among $m$ clients who seek to minimize their global aggregate loss function. The local loss function for client $i$ is denoted as $\tilde{f}_i:\mathbb{R}^n \times\mathcal{D}_i \to \mathbb{R}$, where $\mathcal{D}_i$ denotes the $i$th client's local dataset. To capture incentives, we consider a setting where client $i$ may choose a random subset of $\mathcal{D}_i$ with size $N_i$ when participating in the training task of the federated learning framework. Let $N = (N_1,\ldots, N_m)$ denote the tuple of the sizes of the local training datasets. To this end, we let $a_i(N_i,N_{-i})$ and $c_i(N_i,N_{-i})$ denote the payoff function and cost function of client $i$ associated with the tuple $N$, respectively, where $N_{-i} \triangleq (N_j)_{j\neq i}$.     
This leads to a parameterized federated optimization problem cast as
\begin{equation}\label{prob:incentive_FL}
\begin{aligned}
\min_{x}\quad & \textstyle \sum_{i=1}^{m} p_i(N^*)\, \mathbb{E}_{\xi_i} [\tilde{f}_i(x,\xi_i)] \\
\text{s.t.}\quad 
& N^* \triangleq (N_1^*,\ldots,N_m^*) \text{ solves the game: For all } i, \\
& N_i^* \in \arg\min_{N_i \in \mathcal{N}_i}
\{ c_i(N_i,N_{-i}) - a_i(N_i,N_{-i}) \}.
\end{aligned}
\end{equation}
 where $\mathcal{N}_i$ denotes a local constraint set, $\xi_i \in \mathbb{R}^d$ denotes a random vector associated with the $i$th client's local data, and $p_i (N) \triangleq \frac{N_i}{\sum_{j=1}^m N_j}$ denotes the (unknown) weight of the client $i$. Notably, by construction, these weights are positive and sum to one, i.e., $\sum_{i=1}^m p_i(N)=1$. Further, substituting $\mathbb{E}_{\xi_i}[ \tilde{f}_i(x, \xi_i) ] := \frac{1}{N_i}\sum_{j=1}^{N_i}  \tilde{f}_i(x, \xi_{i,j})$, the global loss function in \eqref{prob:incentive_FL} will account for each data point with an equal weight of $1/\left(\sum_{i=1}^m N_i\right)$. 
Note that when all clients have an equal number of samples, the problem boils down to the standard federated learning formulation of minimizing $\frac{1}{m}\sum_{i=1}^{m}   \mathbb{E}_{\xi_i} [ \tilde{f}_i(x, \xi_i)  ]$.  Throughout, we let $f$ denote the global loss function in \eqref{prob:incentive_FL}, and $f_i(x) \triangleq \mathbb{E}_{\xi_i} [ \tilde{f}_i(x, \xi_i)  ]$ denote the local loss function of client $i$.
\begin{assumption}\label{assum:main}
Consider problem \eqref{prob:incentive_FL}. 
\noindent (i) For any $i \in [m] \triangleq \{1,\ldots,m\}$, the local loss function, $f_i$, is $L$-smooth and the stochastic local function \fm{$\tilde f_i(\bullet, \xi_i)$} is differentiable \fm{for any $\xi_i$}. (ii) For all $i \in [m] $, the local stochastic gradient oracle $\nabla \tilde{f}_i(x,\xi_i)$ is an unbiased estimator of $\nabla f_i$, i.e.,  $\mathbb{E}[\nabla \tilde{f}_i(x,\xi_i) \mid x] = \nabla f_i(x)$, and has a unified bounded variance, i.e., $\mathbb{E}[\|\nabla \tilde{f}_i(x,\xi_i) - \nabla f_i(x)\|^2 \mid x ] \leq \nu^2 $ for some $\nu >0$.  (iii) Let $f^* \triangleq \inf_{x \in \mathbb{R}^n} f(x) > -\infty$, where $f$ denotes the global objective function in \eqref{prob:incentive_FL}. 
\end{assumption}
The constraint set in problem \eqref{prob:incentive_FL} is characterized by a noncooperative Nash game among the clients, where each client $i$ seeks to minimize its utility loss function, denoted by $l_i(N) \triangleq c_i(N) - a_i(N)$, subject to the strategy set $\mathcal{N}_i \triangleq [N_i^{\min}, N_i^{\max}]$, where $N_i^{\min}, N_i^{\max}$ denote the minimum and maximum training size by client $i$, respectively. 
\begin{assumption}[\fr{Strongly monotone data participation}]\label{assum:game}
\noindent (i) For each $i \in [m]$, for any $N_{-i} \in \mathbb{R}^{m-1}$, the utility loss function $l_i(\bullet,N_{-i})$ is differentiable and convex. \fm{(ii)} The mapping $F(N) \triangleq \left(F_i(N)\right)_{i=1}^m$ is $\mu_F$-strongly monotone and $L_F$-Lipschitz continuous, where $F_i(N) \triangleq \nabla_{N_i} l_i(N)$.
\end{assumption}

A data participation profile $N^* = (N_i^*)_{i=1}^m$ is a Nash equilibrium (NE) to the data participation game if no client can unilaterally improve their utility by altering their contribution level, given that all other clients' participation levels remain fixed. Under \fm{Assumption~\ref{assum:game}}, in view of \cite[Prop. 1.4.2]{facchinei2003finite}, the set of all NEs to the data participation game can be characterized by the solution set of the variational inequality problem, $\mbox{VI}(\prod_{i=1}^m\mathcal{N}_i, F)$, defined as $$\left\{N \in \mathcal{N} \mid   F(N)^\top(\bar{N} - N) \geq 0, \hbox{ for all } \bar{N} \in  \mathcal{N} \right\},$$ 
where $\mathcal{N} \triangleq \prod_{i=1}^m\mathcal{N}_i$. Under \fm{Assumption~\ref{assum:game}}, the game admits a unique NE, denoted by $N^*$ (cf. \cite[Ch. 2]{facchinei2003finite}). 

\noindent {\bf Examples of data participation payoff and cost functions.} 

 {\it Random discovery payoff.} Here, it is assumed that any client's payoff is a linear combination of all clients' contributions, i.e., $a(N) \triangleq (a_i(N))_{i=1}^m = W N$ for a symmetric matrix $W \in [0,1]^{m \times m}$ with unit diagonal entries, where $W_{i,j}$ represents the impact of client $j$'s effort on client $i$'s utility. Specifically, for example in a classification setting such as MNIST, each client $i$ maintains a probability distribution $q_i$ over the set of distinct class labels $\mathcal C$ with $|\mathcal C| = t$ (e.g., $t=10$ for MNIST), representing the likelihood that a randomly sampled data point belongs to each class. \fr{Let $q_{ic}$ denote the proportional reward that client $i$ receives} whenever any client samples class $c \in \mathcal C$. The expected reward for client $i$ under the contribution profile $N$ is then $a_i(N) = q_i Q^\top N$, where $Q = [q_{ic}] \in \mathbb{R}_+^{m \times t}$ collects all clients' distributions. Notably, $a(N) = Q Q^\top N$ defines a linear map, and $W = Q Q^\top$.


{\it Cost function.} For $c_i(N_i, N_{-i})$, a natural modeling choice studied in prior work~\cite{blum2021one,murhekar2023incentives} is to assume that it is proportional to the client's contribution, i.e., $c_i(N_i, N_{-i}) \triangleq \lambda_i N_i$, where $\lambda_i > 0$ denotes client $i$'s cost coefficient. We note, however, that an agent’s cost for sharing data may be more complex than just the size of the total data shared, as it can also include losses due to data collection, preprocessing, and privacy concerns~\cite{li2014cost,laudon1996,jaisingh2008}.

\fyr{
\subsection{Merely monotone game: Social welfare maximization}}
\fr{In the preceding formulation, the uniqueness of the Nash equilibrium relies on the strong monotonicity of $F$, which may be restrictive in practice. When $F$ is merely monotone, the game may admit multiple equilibria, and the question of which equilibrium the agents converge to becomes nontrivial. Motivated by this, we consider a setting where $F$ is merely monotone and the agents seek a best Nash equilibrium with respect to a social welfare loss function $h: \mathcal{N} \to \mathbb{R}$. This leads to the parameterized federated optimization problem.}
\fyrr{\begin{equation}\label{prob:incentive_FL_welfare}
\begin{aligned}
\min_{x}\  & \textstyle \sum_{i=1}^{m} p_i(N^*)\, \mathbb{E}_{\xi_i} [\tilde{f}_i(x,\xi_i)] \\
\text{s.t.}\  
& N^* \triangleq [N_j^*]_{j=1}^m \in \text{arg}\min_{N \in \mathcal{N}^*} h(N_1,\ldots,N_m), \\
& \hbox{where } \mathcal{N}^* \triangleq \prod_{i=1}^m \arg\min_{N_i \in \mathcal{N}_i}\{
c_i(N ) - a_i(N )\}.
\end{aligned}
\end{equation}}
\fr{In problem~\eqref{prob:incentive_FL_welfare}, the constraint set is the solution set of the merely monotone Nash game, i.e., $\mathcal{N}^* = \mathrm{SOL}(\mathcal{N}, F)$, and $N^*$ denotes \fyrr{the} minimizer of the welfare loss $h$ over $\mathcal{N}^*$. We address this in the participation update step in Algorithm~\ref{alg:IncentFedAvg} by incorporating an iterative regularization scheme. Specifically, at round $r$, client $i$ updates its participation strategy via
\begin{align*}
\nabla_{i,r,\lambda_r} &:= \nabla_{N_i} c_i(\hat{N}_r) - \nabla_{N_i} a_i(\hat{N}_r) + \lambda_r \nabla_{N_i} h(\hat{N}_r), \\
N_{i,r+1} &:= \Pi_{\mathcal{N}_i}[N_{i,r} - \tilde{\gamma}_r \nabla_{i,r,\lambda_r}],
\end{align*}
The parameter $\tilde{\gamma}_r > 0$ is a diminishing stepsize and
$\lambda_r > 0$ is a diminishing regularization parameter, with the
regularized mapping $F_{\lambda_r}(N) \triangleq F(N) + \lambda_r
\nabla h(N)$ remaining monotone over $\mathcal{N}$. As $\lambda_r \to
0$, the iteratively regularized scheme drives the iterates toward a
minimizer of $h$ over the equilibrium set $\mathcal{N}^*$, as
established in Lemma~\ref{lem:welfare}; setting $\lambda_r \equiv 0$
recovers the unregularized projected-gradient update of the strongly
monotone case, where uniqueness of the equilibrium makes welfare
selection unnecessary. We make the following assumption for the
analysis of problem~\eqref{prob:incentive_FL_welfare}.

\begin{assumption}[Merely monotone data participation]\label{assum:game_monotone}
\noindent (i) For each $i \in [m]$, for any $N_{-i} \in \mathbb{R}^{m-1}$, the utility loss function $l_i(\bullet, N_{-i})$ is differentiable and convex. (ii) The mapping $F(N) \triangleq (F_i(N))_{i=1}^m$ is merely monotone and $L_F$-Lipschitz continuous, where $F_i(N) \triangleq \nabla_{N_i} l_i(N)$. (iii) The welfare loss function $h$ is strictly convex.
\end{assumption}
Under Assumption~\ref{assum:game_monotone}, the solution set $\mathcal{N}^* = \mathrm{SOL}(\mathcal{N}, F)$ is nonempty and convex, and a best Nash equilibrium $N^*$ is well-defined as a minimizer of $h$ over $\mathcal{N}^*$.
}

\section{Algorithm outline}
To address~\eqref{prob:incentive_FL} \fyrr{and \eqref{prob:incentive_FL_welfare}}, we propose \fyrr{a unified} FL scheme termed Incentive-enabled federated averaging, outlined in Algorithm~\ref{alg:IncentFedAvg}. A key challenge is that the weights $p_i(N^*)$, defined by the Nash equilibrium $N^*$, are unavailable. To address this, we develop a federated framework with two coupled components: (i) Cooperative learning: at round $r$, given an estimate $\hat{N}_r$ of $N^*$, clients perform standard local updates as in FedAvg; (ii) Noncooperative game: in the same round, clients update their participation strategy via a projected gradient method. \fr{The server coordinates by broadcasting $(\hat{x}_r, \hat{N}_r)$ at the start of round $r$, and at the end of round $r$, it collects $x_{i,T_{r+1}}$ and the updated participation levels $N_{i,r+1}$ from all clients. It then computes the updated weights $p_{i,r+1} = N_{i,r+1}/\sum_j N_{j,r+1}$ and forms the aggregated model $\hat{x}_{r+1} = \sum_{i=1}^m p_{i,r+1} x_{i,T_{r+1}}$, which serves as the common initialization for round $r+1$.}

\begin{algorithm}
\caption{Incentive-enabled FedAvg (\texttt{IncentFedAvg})}
\begin{algorithmic}[1]
\State \textbf{Input:} A random vector $\hat{x}_0 \in \mathbb{R}^n$, stepsizes $\gamma,\tilde{\gamma}>0$, and synchronization parameters $\{T_0 (:= 0), T_1, \ldots, T_{R}\}$ \fr{(\colorbox{blue!10}{Strongly monotone} and \colorbox{yellow!22}{Social welfare} settings)}; 
\For{each round $r = 0, 1, 2, \ldots, R-1$}
\State Server broadcasts $\hat{x}_r$, i.e., $x_{i,T_r} := \hat{x}_r$, $\forall i \in [m]$
\State Server sends $\hat{N}_r:=(N_{1,r},\dots ,N_{m,r})$ to the clients
\State Client $i$ generates a random training subset of size $\lceil N_{i,r} \rceil$, denoted by $\tilde{\mathcal{D}}_{i,r}$, $\forall i \in [m]$
\For{$k=T_r, \dots, T_{r+1} -1$}
\State Client $i$ randomly samples $\xi_{i,k} \in \tilde{\mathcal{D}}_{i,r}$ 
\State  $x_{i,k+1} := x_{i,k} - \gamma  \nabla\tilde{f}_i(x_{i,k}, \xi_{i,k})$
\EndFor
\State Client $i \in [m]$  updates participation strategy as 

\colorbox{blue!10}{$N_{i,r+1} := \Pi_{\mathcal{N}_i}[N_{i,r} - \tilde{\gamma}(\nabla_{N_i} c_i(\hat{N}_r) -\nabla_{N_i}a_i (\hat{N}_r))]$}
 
\colorbox{yellow!22}{$\fyr{\nabla_{i,r,\lambda_r} := \nabla_{N_i} c_i(\hat{N}_r) -\nabla_{N_i}a_i (\hat{N}_r) +\lambda_{r}\nabla_{N_i} h(\hat{N}_r)}$}
 
\colorbox{yellow!22}{$N_{i,r+1} := \Pi_{\mathcal{N}_i}[N_{i,r} - \fyrr{\tilde{\gamma}_r}\fyr{\nabla_{i,r,\lambda_r}}]$}
\State   Client $ i \in [m]$ sends $x_{i,T_{r+1}}$ and $N_{i, r+1}$ to the server 
\State Server updates weights, $p_{i,r+1} := \frac{\fm{N_{i,r+1}}}{\sum_{i=1}^m \fm{N_{j,r+1}}}$
\State Server's aggregation: $\hat{x}_{r+1} :=  \sum_{i=1}^m \fm{p_{i,r+1}} x_{i,T_{r+1}}$ 
\EndFor
\end{algorithmic}
\label{alg:IncentFedAvg}
\end{algorithm}
\section{Convergence analysis}\label{sec:prob}
In this section, we analyze the convergence and derive   guarantees for addressing problem~\eqref{prob:incentive_FL} in nonconvex and convex settings. Throughout, we let $p_{i,r}\triangleq \tfrac{N_{i,r}}{\textstyle\sum_{j=1}^m N_{j,r}}$ denote the weights, where $N_{i,r}$ is the training sample size by the client $i$ in round $r$. We define $\fm{p^*_i:=p_i( N^*)}$, $p_{\min}^* \triangleq \min_{i\in [m]} p_i^*$, $N^{\max} \triangleq \max_{i \in [m]} N_i^{\max}$,  $N^{\min} \triangleq \min_{i \in [m]} N_i^{\min}$, and assume that $N^{\min} \geq 1$.  
\begin{definition}\label{def:main_terms}
Consider Algorithm \ref{alg:IncentFedAvg}. Let us define  \fm{$g_{i,k} \triangleq \nabla \tilde{f}_i(x_{i,k}, \xi_{i,k})$,   \quad $\bar{g}_k \triangleq  \textstyle\sum_{i=1}^m \fm{p_{i,r}} g_{i,k}$, \quad $\bar{x}_k \triangleq  \textstyle\sum_{i=1}^m \fm{p_{i,r}}  x_{i,k}$, \quad and $\bar{e}_k \triangleq  \textstyle\sum_{i=1}^m \fm{p_{i,r}} \|x_{i,k} - \bar{x}_k\|^2$ for $k \geq 0$.}
\end{definition}
Here, $\bar{x}_k$ is an auxiliary sequence that denotes the average iterates of the clients at any iteration $k$ and $\bar{e}_k$ denotes an average consensus error at that iteration. Wet let $\mathcal{F}_0 = \{\hat{x}_0\}$ and   define
 $
\mathcal{F}_k \triangleq   \cup_{i=1}^m\{\xi_{i,k-1}\} \cup \mathcal{F}_{k-1}, \  \text{for all } k \geq 1.$
\begin{lemma}\label{lemma:agg} 
Consider Algorithm \ref{alg:IncentFedAvg} and Definition \ref{def:main_terms}. For all $k \geq 0$, we have $
\bar{x}_{k+1} = \bar{x}_k - \gamma \bar{g}_k.$
\end{lemma}
In the following result, we show that $\hat{N}_r$ converges linearly to the unique NE, $N^*$. Further, $p_{i,r}$ converges linearly to $p^*_i$.
\begin{lemma}\label{linear_rate} 
Let $N_{i,r}$ be generated by Algorithm~\ref{alg:IncentFedAvg} for all $i \in [m]$ and $r \geq 0$. \fr{Let Assumption~\ref{assum:game} hold.} Suppose $\tilde\gamma \leq \tfrac{\mu_F}{L_F^2}$. The following hold. 

\noindent (i) For all $r\geq 0$, $\|\hat{N}_r -N^* \|\leq (1-0.5\mu_F\tilde\gamma)^r\|\hat{N}_0 - N^*\|.$

\noindent (ii)  For all $i\in [m]$,  
$\fy{\left|p_{i,r} - p_i^* \right| } \leq\delta_{r} \triangleq       \fy{\delta_{0}}(1-0.5\mu_F\tilde\gamma)^r ,$
\fy{where $\delta_{0}\triangleq \left(m-1\right)  \|\hat{N}_0-N^*\| \|N^*\|_1^{-1}$.}
\end{lemma}
\begin{proof}
\noindent {(i)} \fyrr{The proof is provided in Appendix.}

\noindent (ii) Notably, when $m=1$, the result in (ii) holds true, in view of $p_{i,r} = p_i^* =1$. 
For a fixed $i$, we have
\begin{align}\label{eqn:p_i_common_factor}
\left|p_{i,r} - p^*_i \right|= \left|\tfrac{{N_{i,r}}  \textstyle\sum_{j=1}^m N_j^* - N_i^* \textstyle\sum_{j=1}^m N_{j,r}}{(\textstyle\sum_{j=1}^m N_{j,r})(\textstyle\sum_{j=1}^m N_j^*)}\right|.
\end{align}
Rearranging the numerator of the right-hand side in the preceding relation, we have 
\begin{align*}
& N_{i,r} \textstyle\sum_{j=1}^m N_j^* - N_i^* \textstyle\sum_{j=1}^m N_{j,r}   \\ & = N_{i,r} \textstyle\sum_{j \neq i} N_j^* - N_i^* \textstyle\sum_{j \neq i} N_{j,r},
\end{align*}
where $N^*_i N_{i,r}$ is canceled.  Adding and subtracting $N_{i,r}\textstyle\sum_{j \neq i} N_{j,r}$, we obtain
\begin{align*}
&N_{i,r} \textstyle\sum_{j=1}^m N_j^* - N_i^* \textstyle\sum_{j=1}^m N_{j,r} \\ &= N_{i,r} \textstyle\sum_{j \neq i} (N_j^* - N_{j,r}) + (N_{i,r} - N^*_i)\textstyle\sum_{j \neq i} N_{j,r}.
\end{align*}
Invoking the result in part (i), we have for all $i\in [m]$, \fm{$|N_{i,r}-N^*_i| \leq \|\hat{N}_r-N^*\| \leq \|\hat{N}_0-N^*\|\rho^r,$}
where $\rho\triangleq 1-0.5\mu_F\tilde\gamma$.  From the two preceding relations, we obtain
\begin{align*}
& \left|N_{i,r} \textstyle\sum_{j=1}^m N_j^* - N_i^* \textstyle\sum_{j=1}^m N_{j,r}  \right|  \\ & \leq N_{i,r} \textstyle\sum_{j \neq i} \left| N_j^* - N_{j,r} \right| + \left| N_{i,r} - N^*_i \right| \textstyle\sum_{j \neq i} N_{j,r} \\
& \leq N_{i,r} \textstyle\sum_{j \neq i} \|\hat{N}_0-N^*\|\rho^r+ \|\hat{N}_0-N^*\|\rho^r\textstyle\sum_{j \neq i} N_{j,r}\\
& \leq \left(m-1\right)  \|\hat{N}_0-N^*\|\rho^r N_{i,r} \\ &+ (m-1)\|\hat{N}_0-N^*\|\rho^r \textstyle\sum_{j \neq i} N_{j,r}\\ 
&= \left(m-1\right)  \|\hat{N}_0-N^*\|\rho^r \textstyle\sum_{j=1}^m N_{j,r}.
\end{align*}
The result follows by invoking \eqref{eqn:p_i_common_factor}.
\end{proof}

\fyrr{The following lemma extends the previous result to the merely monotone setting.}
\begin{lemma}\label{lem:welfare}
\fr{Let $N_{i,r}$ be generated by Algorithm~\ref{alg:IncentFedAvg} for all
$i \in [m]$ and $r \geq 0$. Let Assumption~\ref{assum:game_monotone} hold.
Let $C_F \triangleq \sup_{N \in \mathcal{N}}\|F(N)\|$,
$C_H \triangleq \sup_{N \in \mathcal{N}}\|\nabla h(N)\|$, and
$D_{\mathcal{N}}^2 \triangleq \sup_{N',N'' \in \mathcal{N}}
\tfrac{1}{2}\|N' - N''\|^2$, all finite by the compactness of
$\mathcal{N}$.} \fr{Let $\{\tilde{\gamma}_r\}_{r\ge 0}$ and
$\{\lambda_r\}_{r\ge 0}$ be positive sequences with $\{\lambda_r\}$
nonincreasing. Define the weighted-average iterate
weights}
\fyrr{$
\bar{N}_R \triangleq
{\sum_{r=0}^{R-1}\tilde{\gamma}_r\lambda_r\,\hat{N}_r}
     /{\left(\sum_{j=0}^{R-1}\tilde{\gamma}_j\lambda_j\right)}.$
  Then:}
 
\fr{\noindent (i) [Welfare-gap upper bound] For all $R\ge 1$,
\begin{align*}
h(\bar{N}_R) - h(N^*)
\;\le\;
\frac{2 D_{\mathcal{N}}^2
+ (C_F + \lambda_0 C_H)^2 \sum_{r=0}^{R-1}\tilde{\gamma}_r^2}
{2\sum_{r=0}^{R-1}\tilde{\gamma}_r\lambda_r},
\end{align*}} \fyrr{and $
h(\bar{N}_R) - h(N^*) \ge -C_H \,\mathrm{dist}\!\left(\bar{N}_R, \mathrm{SOL}(\mathcal{N}, F)\right).$}

\fr{\noindent (ii) [NE-infeasibility bound]
For all $R \geq 1$,
\begin{align*}
&0 \leq \mathrm{Gap}(\bar{N}_R, \mathcal{N}, F)
\leq \tfrac{\lambda_0 D_{\mathcal{N}}^2}{\sum_{r=0}^{R-1}\tilde\gamma_r\lambda_r}+\tfrac{(C_F + \lambda_0 C_H)^2 \sum_{r=0}^{R-1}\tilde{\gamma}_r^2 \lambda_r}{2\sum_{r=0}^{R-1}\tilde\gamma_r\lambda_r}+\tfrac{ C_H\sqrt{2} D_{\mathcal N}\sum_{r=0}^{R-1}\tilde\gamma_r \lambda_r^2}{\sum_{r=0}^{R-1}\tilde\gamma_r\lambda_r},
\end{align*}}
 \fyrr{where $\mathrm{Gap}(\bullet, \mathcal N, F)$ denotes the dual gap function associated with $\mbox{VI}(\mathcal{N},F)$.}

\fr{\noindent (iii) Suppose $\sum_{r=0}^{\infty}\tilde{\gamma}_r\lambda_r = \infty$ as $R\to\infty$, 
 $\dfrac{\sum_{r=0}^{R-1}\tilde\gamma_r\lambda_r^2}{\sum_{r=0}^{R-1}\tilde\gamma_r\lambda_r}\to 0$,
and $\dfrac{\sum_{r=0}^{R-1}\tilde\gamma_r^2, }{\sum_{r=0}^{R-1}\tilde\gamma_r\lambda_r}\to 0$. 
\fyrr{Then, $\lim_{R\to \infty}\bar{N}_R$ exists and is equal to $N^*$.} In particular,
the stepsize and the regularization \fyrr{parameter
$
\tilde\gamma_r =  {\tilde\gamma_0}{(r+1)^{-a}},
\lambda_r =  {\lambda_0}{(r+1)^{-b}},$ where
$0 < b < a$, $a + b < 1,$  satisfy} the three rate conditions.}

\noindent \fyrr{(iv) For all $i \in [m]$, $|p_{i,r}-p_i^*|\leq \hat{\delta}_r \triangleq \delta_0\|\hat{N}_r-N^*\|$ where $\delta_{0}$ is given by Lemma~\ref{linear_rate}. Further, if the limit of $\hat{N}_r$ exists, then $\lim_{r\to \infty}|p_{i,r}-p_{i}^*| =0$ for all $i \in [m]$.}
\end{lemma}
 
\begin{proof}
\fr{Fix an arbitrary $N \in \mathcal{N}$. From the projection theorem, \fyrr{noting that $\hat{N}_{r+1} = \Pi_{\mathcal{N}}[u]$, we have}
\begin{align*}
(\hat{N}_{r+1} - \hat{N}_r + \tilde{\gamma}_r(F(\hat{N}_r)
+ \lambda_r\nabla h(\hat{N}_r)))^\top(N - \hat{N}_{r+1}) \geq 0,
\end{align*}
which rearranges to
\begin{align}\label{eqn:proj_rearranged}
&\tilde{\gamma}_r(F(\hat{N}_r) + \lambda_r\nabla h(\hat{N}_r))^\top
(\hat{N}_{r+1} - N)
\notag\\
&\leq (\hat{N}_r - \hat{N}_{r+1})^\top(\hat{N}_{r+1} - N).
\end{align}
Using $2a^\top b = \|a+b\|^2 - \|a\|^2 - \|b\|^2$ with
$a := \hat{N}_r - \hat{N}_{r+1}$, $b := \hat{N}_{r+1} - N$,
splitting $\hat{N}_{r+1}-N = (\hat N_r - N) - (\hat N_r-\hat N_{r+1})$ on
the left of \eqref{eqn:proj_rearranged}, and applying Young's inequality
$2a^\top b\le\|a\|^2+\|b\|^2$ to the cross-term yields, after the
$\|\hat{N}_r - \hat{N}_{r+1}\|^2$ terms cancel,
\begin{align}\label{eqn:after_young}
&2\tilde{\gamma}_r(F(\hat{N}_r) + \lambda_r\nabla h(\hat{N}_r))^\top
(\hat{N}_r - N)
\leq \|\hat{N}_r - N\|^2
\notag\\
&- \|\hat{N}_{r+1} - N\|^2
+ \tilde{\gamma}_r^2\|F(\hat{N}_r) + \lambda_r\nabla h(\hat{N}_r)\|^2.
\end{align}
By the triangle inequality, $\|F(\hat N_r)+\lambda_r\nabla h(\hat N_r)\|
\le C_F + \lambda_r C_H \le C_F + \lambda_0 C_H$, so
\begin{align}\label{eqn:compact_bound}
\tilde{\gamma}_r^2\|F(\hat{N}_r) + \lambda_r\nabla h(\hat{N}_r)\|^2
\le \tilde{\gamma}_r^2(C_F + \lambda_0 C_H)^2.
\end{align}
}
 
\fr{\noindent(i)
Set $N := N^*$ in \eqref{eqn:after_young}. Combining with
\eqref{eqn:compact_bound},
\begin{align}\label{eqn:8_at_Nstar}
&2\tilde{\gamma}_r(F(\hat{N}_r) + \lambda_r\nabla h(\hat{N}_r))^\top
(\hat{N}_r - N^*)
\le \|\hat{N}_r - N^*\|^2\notag\\
&- \|\hat{N}_{r+1} - N^*\|^2
+ \tilde{\gamma}_r^2(C_F + \lambda_0 C_H)^2.
\end{align}
Expanding the left-hand side as
$2\tilde\gamma_r F(\hat{N}_r)^\top(\hat{N}_r - N^*)
+ 2\tilde\gamma_r\lambda_r \nabla h(\hat{N}_r)^\top(\hat{N}_r - N^*)$,
using monotonicity of $F$ and $N^* \in \mathcal{N}^* \subseteq
\mathrm{SOL}(\mathcal{N},F)$ to obtain
$F(\hat{N}_r)^\top(\hat{N}_r - N^*) \ge 0$, and using convexity of $h$
to obtain $\nabla h(\hat{N}_r)^\top(\hat{N}_r - N^*) \ge h(\hat{N}_r)
-h(N^*)$, we drop the nonnegative term \fyrr{and} obtain
\begin{align}\label{eqn:per-round-h}
&2\tilde{\gamma}_r\lambda_r(h(\hat{N}_r) - h(N^*))
\le \|\hat{N}_r - N^*\|^2
- \|\hat{N}_{r+1} - N^*\|^2
\notag\\
&+ \tilde{\gamma}_r^2(C_F + \lambda_0 C_H)^2.
\end{align}
Summing \eqref{eqn:per-round-h} over $r=0,\ldots,R-1$, the norm squared
terms telescope. Using $\|\hat N_0 - N^*\|^2 \le 2 D_{\mathcal{N}}^2$, \fyrr{we obtain}
\begin{align*}
&2\textstyle\sum_{r=0}^{R-1}\tilde{\gamma}_r\lambda_r(h(\hat{N}_r) - h(N^*))
\le 2 D_{\mathcal{N}}^2
\\
&+ (C_F + \lambda_0 C_H)^2\textstyle\sum_{r=0}^{R-1}\tilde{\gamma}_r^2.
\end{align*}
Dividing both sides by $2\sum_{r=0}^{R-1}\tilde{\gamma}_r\lambda_r > 0$
and applying Jensen's inequality for the convex function $h$ with
weights $w_r \triangleq \tilde{\gamma}_r\lambda_r/
\sum_{s=0}^{R-1}\tilde{\gamma}_s\lambda_s$ yields the bound in (i).
}

\fr{For the lower bound, let $\hat N$ denote the projection of
$\bar N_R$ onto $\mathrm{SOL}(\mathcal{N}, \fyrr{F})$, so $\|\bar N_R - \hat N\|
= \mathrm{dist}(\bar N_R, \mathrm{SOL}(\mathcal{N}, \fyrr{F}))$. By convexity
of $h$ at $N^*$,
\begin{align*}
&h(\bar N_R) - h(N^*)
\geq \nabla h(N^*)^\top (\bar N_R - N^*) \\
&= \nabla h(N^*)^\top (\bar N_R - \hat N)
+ \nabla h(N^*)^\top (\hat N - N^*).
\end{align*}
Since $N^* \in \mathrm{SOL}(\mathcal{N}, \fyrr{F})$ and $\hat N \in \mathcal{N}$,
the \fyrr{optimality} condition at $N^*$ gives
$\nabla h(N^*)^\top (\hat N - N^*) \ge 0$, so this term can be dropped.
Bounding the remaining term by Cauchy--Schwarz with $\|\nabla h(N^*)\|
\le C_H$,
\begin{align*}
&h(\bar N_R) - h(N^*)
\geq -C_H \|\bar N_R - \hat N\|
\\
&= -C_H \,\mathrm{dist}(\bar N_R, \mathrm{SOL}(\mathcal{N}, \fyrr{F})),
\end{align*}
which is the lower bound claimed in (i).
}
 
\fr{\noindent (ii)
\fyrr{Consider} \eqref{eqn:after_young} with $N$ arbitrary in $\mathcal{N}$.
Combining \eqref{eqn:after_young} and \eqref{eqn:compact_bound} and
expanding the left-hand side,
\begin{align*}
&2\tilde\gamma_r F(\hat N_r)^\top(\hat N_r - N)
+ 2\tilde\gamma_r\lambda_r \nabla h(\hat N_r)^\top(\hat N_r - N)
\\
&\leq \|\hat{N}_r - N\|^2
- \|\hat{N}_{r+1} - N\|^2
+ \tilde{\gamma}_r^2(C_F + \lambda_0 C_H)^2.
\end{align*}
By monotonicity of $F$,
$F(\hat N_r)^\top(\hat N_r - N) \ge F(N)^\top(\hat N_r - N)$.
By \fyrr{the} Cauchy--Schwarz \fyrr{inequality}, with $\|\nabla h(\hat N_r)\|\le C_H$ and
$\|\hat N_r - N\|\le \sqrt{2}D_{\mathcal N}$,
$2\tilde\gamma_r\lambda_r \nabla h(\hat N_r)^\top(\hat N_r - N)
\ge -2\tilde\gamma_r\lambda_r C_H \sqrt{2} D_{\mathcal N}$. Substituting
both and rearranging,
\begin{align}\label{eqn:F-perround}
&2\tilde\gamma_r F(N)^\top(\hat N_r - N)
\le
\|\hat{N}_r - N\|^2
- \|\hat{N}_{r+1} - N\|^2
\notag\\
&+ \tilde{\gamma}_r^2(C_F + \lambda_0 C_H)^2
+ 2\tilde\gamma_r\lambda_r C_H \sqrt{2} D_{\mathcal N}.
\end{align}
Multiply both sides of \eqref{eqn:F-perround} by $\lambda_r > 0$,
\begin{align}\label{eqn:F-times-lambda}
&2\tilde\gamma_r\lambda_r F(N)^\top(\hat N_r - N)
\le
\lambda_r\|\hat N_r - N\|^2
- \lambda_r\|\hat N_{r+1} - N\|^2
\notag\\
&+ \tilde\gamma_r^2 \lambda_r (C_F + \lambda_0 C_H)^2
+ 2\tilde\gamma_r \lambda_r^2 C_H \sqrt{2} D_{\mathcal N}.
\end{align}
Since $\{\lambda_r\}$ is nonincreasing,
$-\lambda_r\|\hat N_{r+1} - N\|^2 \le -\lambda_{r+1}\|\hat N_{r+1}-N\|^2$.
Substituting this into \eqref{eqn:F-times-lambda} and summing
$r = 0,\ldots,R-1$, the norm squared terms telescope to
$\lambda_0 \|\hat N_0 - N\|^2 - \lambda_R \|\hat N_R - N\|^2$.
Dropping $-\lambda_R \|\hat N_R - N\|^2 \le 0$ and using
$\|\hat N_0 - N\|^2 \le 2 D_{\mathcal N}^2$, we obtain
\begin{align*}
&2 F(N)^\top \textstyle\sum_{r=0}^{R-1}\tilde\gamma_r\lambda_r(\hat N_r - N)
\leq
2\lambda_0 D_{\mathcal{N}}^2
\\
&+ (C_F + \lambda_0 C_H)^2\textstyle\sum_{r=0}^{R-1}\tilde\gamma_r^2 \lambda_r
+ 2 C_H\sqrt{2} D_{\mathcal N}\textstyle\sum_{r=0}^{R-1}\tilde\gamma_r \lambda_r^2.
\end{align*}
Dividing by $2\sum_{r=0}^{R-1}\tilde\gamma_r\lambda_r > 0$ and recognizing
the average $\bar N_R = \sum \tilde\gamma_r\lambda_r \hat N_r/
\sum\tilde\gamma_r\lambda_r$,
\begin{align*}
&F(N)^\top(\bar N_R - N)
\leq \tfrac{\lambda_0 D_{\mathcal{N}}^2}{\sum_{r=0}^{R-1}\tilde\gamma_r\lambda_r}\\
&+\tfrac{(C_F + \lambda_0 C_H)^2 \sum_{r=0}^{R-1}\tilde{\gamma}_r^2 \lambda_r}{2\sum_{r=0}^{R-1}\tilde\gamma_r\lambda_r}+\tfrac{ C_H\sqrt{2} D_{\mathcal N}\sum_{r=0}^{R-1}\tilde\gamma_r \lambda_r^2}{\sum_{r=0}^{R-1}\tilde\gamma_r\lambda_r}.
\end{align*}
The right-hand side is independent of $N$. Taking the supremum over
$N \in \mathcal{N}$ on the left and invoking the definition
$\mathrm{Gap}(\bar N_R, \mathcal N, F) = \sup_{N\in\mathcal N}
F(N)^\top(\bar N_R - N)$ yields the upper bound stated in (ii).
The lower bound $\mathrm{Gap}(\bar N_R, \mathcal{N}, F) \ge 0$
holds because $\bar N_R \in \mathcal{N}$ (a convex combination of
$\hat N_r \in \mathcal{N}$). Taking $N := \bar N_R$ inside the supremum
defining the gap gives $F(\bar N_R)^\top(\bar N_R - \bar N_R) = 0$,
so the supremum is at least zero.
}
 
\fr{\noindent(iii)
Note that $\bar N_R \in \mathcal{N}$ for all $R\ge 1$ (a convex
combination of $\hat N_r \in \mathcal N$) and $\mathcal N$ is compact.
Invoking the Bolzano--Weierstrass theorem, $\{\bar N_R\}$ has at least
one accumulation point. Let us denote an arbitrary convergent
subsequence of $\{\bar N_R\}$ by $\{\bar N_{R_i}\}$ and let $\hat N$
denote the accumulation point.
}

\fr{Recall that $\mathrm{Gap}(\bullet, \mathcal{N},F)$ is a continuous
function \cite[Ch.~2]{facchinei2003finite}. Under the rate conditions
assumed in (iii), each of the three terms in the upper bound of (ii)
tends to zero as $R\to\infty$. \fyrr{By taking} the limit along
$\{\bar N_{R_i}\}$ we obtain
$\mathrm{Gap}(\hat N, \mathcal{N}, F) = 0$. Invoking
\cite[Prop.~2.3.15]{facchinei2003finite}, we have $\hat N \in
\mathcal{N}^*$.
}

\fr{From part (i), the right-hand side of the welfare-gap bound also
tends to zero under the rate conditions in (iii). The first term
$2D_{\mathcal N}^2/(2\sum_{r=0}^{R-1}\tilde\gamma_r\lambda_r) \to 0$ by
$\sum_{r=0}^{\infty}\tilde\gamma_r\lambda_r=\infty$, and the second
term $(C_F+\lambda_0 C_H)^2 \sum_{r=0}^{R-1}\tilde\gamma_r^2/
(2\sum_{r=0}^{R-1}\tilde\gamma_r\lambda_r) \to 0$ by the second rate
condition. The function $h$ is continuous, so taking the limit along
$\{\bar N_{R_i}\}$ on the bound in (i) yields $h(\hat N) - h(N^*) \le 0$,
i.e., $h(\hat N) \le h(N^*)$. Combined with $\hat N \in \mathcal{N}^*$,
this gives $\hat N \in \arg\min_{N \in \mathcal{N}^*} h(N)$, which
establishes the convergence claim in (iii).
}

\noindent \fyrr{(iv) The proof of $|p_{i,r}-p_i^*|\leq \hat{\delta}_r \triangleq \delta_0\|\hat{N}_r-N^*\|$ can be done in a similar vein to the proof in  Lemma~\ref{linear_rate} (ii) and is omitted. To show the asymptotic result, note that since the limit of $\hat{N}_r$ exists, invoking the Ces\`aro mean theorem and that the weighted average iterate $\bar{N}_r$ converges to $N^*$, we have that $\hat{N}_r \to N^*$ as $r \to \infty$. Thus, $\lim_{r\to \infty }|p_{i,r}-p_i^*|\leq  \delta_0\lim_{r\to \infty }\|\hat{N}_r-N^*\| =0$.}
\end{proof}

\subsection{Main results \fr{(Strongly monotone game)}} In this section, we present the convergence guarantees for addressing the problem~\eqref{prob:incentive_FL} in both convex and nonconvex settings. In the nonconvex setting, we make the following assumption, which has been utilized in the analysis of the federated averaging method in the non-iid setting~\cite{karimireddy2020scaffold}.
\begin{assumption}[Bounded gradient dissimilarity]\label{assum:bgd} 
There exist constants $G \geq 0$ and $B \geq 0$ such that    \fy{$
\tfrac{1}{m}\textstyle \sum_{i=1}^{m} \| \nabla f_i(x)\|^2 \leq G^2 + B^2\|\nabla f(x)\|^2, $ for all $x \in \mathbb{R}^n$.}
\end{assumption} 
\begin{theorem}[\fr{Nonconvex loss, strongly monotone game}] \label{Theorem:main}  
Consider Algorithm \ref{alg:IncentFedAvg} \fyrr{under} Assumptions \ref{assum:main}, \ref{assum:game},  and \ref{assum:bgd}.

\noindent (i) {[Error bounds]} Suppose $\gamma \leq \left(\min\left\{\frac{1}{32}, \tfrac{\sqrt{2}}{\sqrt{153} m B H  }\right\}\right)\frac{1}{L}$, $\tilde\gamma \leq \frac{\mu_F}{L_F^2}$, and $\hat{r}$ is an integer such that \fm{$   \tfrac{\ln(17m^2\delta_{0}^2  B^2)}{2\ln(1/\rho)} \leq  \hat{r}\leq R-1$} where $\rho=(1-0.5\mu_F\tilde\gamma)$ and $R\geq 1$. Let $k^*$ denote an integer drawn uniformly at random from $\{T_{\hat{r}},\ldots,T_R-1\}$. 
 Then, the following hold. 

\noindent (i-1) [Optimality bound] We have 
\begin{align*}
\mathbb{E}[\|\nabla f(\bar{x}_{k^*})\|^2] &\leq 
16(\gamma (T_R-T_{\hat r}))^{-1}(\mathbb{E}[f(\bar{x}_{\hat{r}})] - f^*)\\
&+  8 L \gamma  \nu^2 +  \tfrac{153}{2} \left(\nu^2 + mG^2\right) m L^2 H^2\gamma^2 \\
&+   16  m^2 \fy{\delta_{0}}^2  G^2 \left(1+2L\gamma \right)  \tfrac{\rho^{2\hat{r}}-\rho^{2R}}{(1-\rho^2)(R-\hat{r})} .
\end{align*}

\noindent (i-2) [Equilibrium infeasibility bound] Let $r^*$ denote the  round index associated with $k^*$.  We have 
\begin{align*} \mathbb{E}[\|\hat{N}_{r^*}-N^*\|^2] \leq  \tfrac{\rho^{2\hat r}-\rho^{2R}}{(1-\rho^2)(R-\hat{r})}\|\hat{N}_0 - N^*\|^2.\end{align*}

\noindent (ii)  [\fy{Communication complexity}] Let $\varepsilon > 0$ be an arbitrary scalar and $R_{\varepsilon}$ denote the number of communication rounds such that $\max\{\mathbb{E}[\|\nabla f(\bar{x}_{k^*})\|^2],\mathbb{E}[\|\hat{N}_{r^*}-N^*\|^2]\} \leq \varepsilon.$\\
Suppose $\gamma := \sqrt{\frac{1}{R_{\varepsilon} H}}$, $H$ is a constant, $T_R:= R_\varepsilon H$, and $T_{\hat{r}}:=\frac{R_\varepsilon H}{2}$. Assume that $\hat{f} \triangleq \sup_{x \in \mathbb{R}^n} f(x) < \infty$. Then, $R_\varepsilon = \mathcal{O}\left( \max\{   \fy{R_{1,\varepsilon}} ,  \fy{R_{2,\varepsilon}},\fy{R_{3}},  \fy{H(m BL)^2} , \frac{L^2}{H}\}\right)$ where we define

 $ \fy{R_{1,\varepsilon}}\triangleq \tfrac{(\hat{f} - f^*)^2+L^2\nu^4+m^2L^4H^4(\nu^2+mG^2)^2}{\fy{H}\varepsilon^2},$ $\fy{R_{2,\varepsilon}} \triangleq \tfrac{\fy{1}}{\ln(1/\rho)}\ln\left(\tfrac{ \|\hat{N}_0-N^*\|^2   \max\{\sqrt{192} m  (m-1) \|N^*\|_1^{-1} G,1\}^2   }{ (1-\rho^2)L^2 \,\varepsilon}\right),$ and $ \fy{R_3}\triangleq \tfrac{ \ln(17 m^2 \fy{\delta_{0}}^2 B^2)}{\ln(1/\rho)}$.
\end{theorem}
\begin{remark}
   Thm.~\ref{Theorem:main} \fyrr{establishes simultaneous error bounds} for both stationary point computation and participation equilibrium infeasibility, along with iteration and communication complexity guarantees. 
\end{remark}
Next, we extend our analysis to the case where the local loss functions are convex. This setting allows for convergence to a global optimal solution rather than just a stationary point.
\begin{assumption}\label{assum:c} 
Consider problem \eqref{prob:incentive_FL}. \fy{Let Assumption \ref{assum:main} (i-ii) hold. Further, suppose, for any $i \in [m]$, ${f}_i$ is convex  and $\arg\min_{x \in \mathbb{R}^n} f_i(x)$ is nonempty. Also, suppose $X^* \neq \emptyset$.} 
\end{assumption}
\begin{theorem}[\fr{Convex loss, strongly monotone game}]\label{thm:convex} 
  Consider Algorithm~\ref{alg:IncentFedAvg}. Let Assumptions~\ref{assum:game} and \ref{assum:c} hold.

\noindent (i) [Error bounds] Let $\bar{x}_T^{\text{avg}} = \frac{1}{T_R-T_{\hat{r}}}\sum_{k=T_{\hat{r}}}^{T_R-1}\bar{x}_k$. If    $ \gamma \leq \fy{\frac{1}{mH\sqrt{\delta_{0}}(1-\rho^{0.5})^{-1}+24LHB\sqrt{m}}},$ and $\gamma \leq \min\left\{\frac{1}{16 L},\frac{1}{\fy{2}mL},  \frac{1}{16Lm^2B^2\fy{\delta_{0}}}, \frac{1}{\sqrt{3}HL}\right\}$, and $\frac{\ln(\fy{\delta_{0}}\sqrt[3]{16L^2B^4})}{\ln(1/\rho)} \leq\hat{r} \leq R-1 $, then, the following hold. 

\noindent (i-1) [Optimality bound] We have
\begin{align*}
&\mathbb{E}[f(\bar{x}_T^{\text{avg}})] - f^* \leq \left(1-\tfrac{\fm{m}\gamma \fm{\sqrt{\fy{\delta_{0}}}} H}{1-\rho^{0.5}}\right)^{-1}\left(\tfrac{4\fm{\mathbb{E}[\|\bar{x}_{T_{\hat{r}}}-x^*\|^2]}}{\gamma (T_R-T_{\hat{r}})} \right. \\ 
& \left. + \tfrac{72\gamma^2 L(\nu^2 + mG^2)H^3}{(T_R-T_{\hat{r}})} + \fm{4\gamma\nu^2} \right.\\
& \left. + \tfrac{\fm{16}\gamma m^2\fm{\fy{\delta_{0}}^2}G^2 H}{(T_R-T_{\hat{r}})(1-\rho^2)} + \tfrac{4\fm{\fy{\delta_{0}}\sqrt{\fy{\delta_{0}}}(G^2 + \sum_{i=1}^m \|\nabla f_i(x^*)\|^2)} H}{(T_R-T_{\hat{r}})(1-\rho^{1.5})} \right).
\end{align*}
\noindent (i-2) [Equilibrium infeasibility bound]  See Thm.~\ref{Theorem:main}. (i-2).\\
\noindent (ii) [Communication complexity] \fm{Let $\varepsilon > 0$ be an arbitrary scalar and $R_\varepsilon$ denote the number of communication rounds such that $\mathbb{E}[f(\bar{x}_T^{\text{avg}})] - f^* \leq \varepsilon$, and let $D_{T_{\hat{r}}}^2 := \mathbb{E}[\|\bar{x}_{T_{\hat{r}}}-x^*\|^2]$. Suppose $\gamma := (R_\varepsilon H)^{-1/2}$ and $T_R := R_\varepsilon H$.} Then, \fm{$R_{\varepsilon} = \mathcal{O}\left(\max\{R_{1,\varepsilon}, R_{2,\varepsilon}, R_{3,\varepsilon}, R_{4,\varepsilon} ,  \hat{R}\right)$, where we define} $ R_{1,\varepsilon} \triangleq \tfrac{D_{T_{\hat{r}}}^4+\nu^4}{H\varepsilon^2},$  $R_{2,\varepsilon} \triangleq \sqrt{\tfrac{L(\nu^2 + mG^2)}{H\varepsilon}} ,$    $R_{3,\varepsilon} \triangleq \tfrac{1}{H^{1/3}}\left(\tfrac{\fm{ m^2\fy{\delta_{0}}^2 G^2}}{(1-\rho^2)\varepsilon}\right)^{2/3},$  $R_{4,\varepsilon} \triangleq \tfrac{\fm{\fy{\delta_{0}}\sqrt{\fy{\delta_{0}}}(G^2 + \sum_{i=1}^m \|\nabla f_i(x^*)\|^2)}}{(1-\rho^{1.5})\varepsilon},$ and $\hat{R}$ is a sufficiently large constant ensuring that the conditions on $\gamma$ and $R$ in part (i) are satisfied. 
\end{theorem}

\fr{\subsection{Main results (Merely monotone game)}}
\fyrr{Here, we present the convergence guarantees for 
addressing problem~\eqref{prob:incentive_FL_welfare} under the merely 
monotone data participation game. We adopt the 
iteratively regularized scheme in
Algorithm~\ref{alg:IncentFedAvg} associated with the social welfare setting.}  

\fyrr{
\begin{theorem}[Nonconvex loss, merely monotone game]
\label{thm:nonconvex_mm}
Consider Algorithm~\ref{alg:IncentFedAvg}. Let
Assumptions~\ref{assum:main}, \ref{assum:game_monotone}, and
\ref{assum:bgd} hold and assume that $\hat{N}_r$ is convergent.  Suppose the
participation stepsize and regularization parameter are
$\tilde\gamma_r = \tilde\gamma_0(r+1)^{-a}$ and
$\lambda_r = \lambda_0(r+1)^{-b}$ with $0<b<a<1$ and $a+b<1$. Suppose $
\gamma \le \min\left\{
\tfrac{1}{32 L},\;
\tfrac{\sqrt{2}}{\sqrt{153}\,mBHL}
\right\}.$
Then the following hold.\\

\noindent (i) [Error bounds] There exists an integer $\hat{r}$ such that for $k^*$ uniformly drawn at random from $\{T_{\hat{r}},\ldots,T_R-1\}$, we have 
\begin{align*}
  \mathbb{E}[\|\nabla f(\bar{x}_{k^*})\|^2]  &
  \leq 16 (\gamma (T_R-T_{\hat r}))^{-1}(\mathbb{E}[f(\bar{x}_{\hat{r}})] - f^*)\\& +  {8 L \gamma  \nu^2} 
+  \tfrac{153}{2} (\nu^2 + mG^2) m L^2 H^2\gamma^2 \\
&+    16 m^2    G^2 (1+2L\gamma )  \tfrac{1}{T_R-T_{\hat r}}\textstyle\sum_{k=T_{\hat{r}}}^{T_R-1} \hat{\delta}_{r(k)}^2 .
\end{align*}

\noindent (ii) Let $\varepsilon>0$ be an arbitrary scalar. Let
$\gamma \le \min\left\{ \tfrac{\varepsilon}{16 L \nu^2},\;\sqrt{\tfrac{2\varepsilon}{153(\nu^2+mG^2)\,mL^2 H^2}}\right\}$. Then, we have 
\begin{align*}
\limsup_{R\to\infty}\,
\mathbb{E}[\|\nabla f(\bar{x}_{k^*})\|^2]
\;\le\; \varepsilon.
\end{align*}

\end{theorem}}

\subsection{Analysis}
\begin{lemma}\label{lem:weighted_gradient_inner_product}
Let $\bar{g}_t  $ be given by Def.~\ref{def:main_terms}. Then, for $t\geq 0$,
\begin{align*}
&\textstyle\sum_{i=1}^{m} p_{i,r} \mathbb{E}\left[(g_{i,t})^\top \bar{g}_t | \mathcal{F}_{T_r}\right] = \mathbb{E}\left[\|\bar{g}_t\|^2 | \mathcal{F}_{T_r}\right].
\end{align*}
\end{lemma}

\begin{lemma}\label{lem:agg_gradient}
Let Assumptions~\ref{assum:main}, \ref{assum:game}, and \ref{assum:bgd} hold. Consider Algorithm~\ref{alg:IncentFedAvg}. Then, for any $k \geq 1$,
\begin{align*}
 \mathbb{E} \left[\nabla f(\bar{x}_k)^{\top} \bar{g}_k \mid \mathcal{F}_k\right] &\geq -  m^2  \fy{\delta_r}^2 G^2     - \tfrac{m L^2 }{2}  \bar{e}_k \\ 
 & +\left(\tfrac{1}{4} -  m^2\fy{\delta_r}^2 B^2 \right)\|\nabla f(\bar{x}_k)\|^2.
\end{align*}
\end{lemma}
\begin{proof}
From Assumption~\ref{assum:main} and Def.~\ref{def:main_terms}, we have
\fm{\begin{align*}
&\mathbb{E} \left[\nabla f(\bar{x}_k)^{\top} \bar{g}_k \mid \mathcal{F}_k\right] = \nabla f(\bar{x}_k)^{\top} \mathbb{E} \left[\textstyle\sum_{i=1}^m \fm{p_{i,r}} g_{i,k} \mid \mathcal{F}_k\right] \\ 
& = \nabla f(\bar{x}_k)^{\top} \textstyle \sum_{i=1}^m \fm{p_{i,r}} \nabla f_i(x_{i,k}).
\end{align*}} 
Adding and subtracting $\nabla f(\bar{x}_k)^{\top} \textstyle \sum_{i=1}^m \fm{p_{i,r}} \nabla f_i(\bar{x}_k)$ and $\nabla f(\bar{x}_k)^{\top} \textstyle \sum_{i=1}^m \fm{p_{i}^*} \nabla f_i(\bar{x}_k)$, we obtain
\begin{align}\label{Main_Term}
& \mathbb{E} \left[\nabla f(\bar{x}_k)^{\top} \bar{g}_k \mid \mathcal{F}_k\right] \notag
\\ & = \fm{\nabla f(\bar{x}_k)^{\top}\textstyle \sum_{i=1}^m p_{i,r} (\nabla f_i(x_{i,k}) - \nabla f_i(\bar{x}_k))}\notag\\
&+ \fm{\nabla f(\bar{x}_k)^{\top}\textstyle \sum_{i=1}^m (p_{i,r}-p^*_i) \nabla f_i(\bar{x}_k) +\|\nabla f(\bar{x}_k)\|^2,}
\end{align} 
where we utilized $\nabla f(\bar{x}_k) = \sum_{i=1}^m \fm{p^*_i} \nabla f_i(\bar{x}_k)$. First, we derive a lower bound on the first term on the right-hand side. Using the identity
$a^{\top} b \geq -\frac{1}{2}\|a\|^2 - \frac{1}{2}\|b\|^2$, for $a := \nabla f(\bar{x}_k)$ and $b := \sum_{i=1}^m \fm{p_{i,r}}(\nabla f_i(x_{i,k}) - \nabla f_i(\bar{x}_k))$, we obtain
\fm{\begin{align*}
 & \fm{\nabla f(\bar{x}_k)^{\top} \textstyle\sum_{i=1}^m p_{i,r} (\nabla f_i(x_{i,k}) - \nabla f_i(\bar{x}_k))}  \\&\geq -\tfrac{1}{2} {\textstyle\|\nabla f(\bar{x}_k)\|^2}   -\tfrac{1}{2}{ \left\| \textstyle\sum_{i=1}^m \fm{p_{i,r}}(\nabla f_i(x_{i,k}) - \nabla f_i(\bar{x}_k))\right\|^2}.
\end{align*}}
Invoking the identity $\|\textstyle\sum_{t=1}^{T}y_t \|^2 \leq T \textstyle\sum_{t=1}^{T}\|y_t\|^2$,  
\fm{\begin{align*}
 & \fm{\nabla f(\bar{x}_k)^{\top} \textstyle\sum_{i=1}^m p_{i,r} (\nabla f_i(x_{i,k}) - \nabla f_i(\bar{x}_k))}  \\ & \geq -\tfrac{1}{2}\|\nabla f(\bar{x}_k)\|^2 - \tfrac{m}{2} \textstyle\sum_{i=1}^m \fm{p_{i,r}^2} \|\nabla f_i(x_{i,k}) - \nabla f_i(\bar{x}_k)\|^2.
\end{align*}}
Invoking the Lipschitz continuity of the local gradients and $p_{i,r} \leq 1$, we have
\begin{align}\label{First_Term}
  &\fm{\nabla f(\bar{x}_k)^{\top} \textstyle\sum_{i=1}^m p_{i,r} (\nabla f_i(x_{i,k}) - \nabla f_i(\bar{x}_k))} \notag 
\\ 
&\geq -\tfrac{1}{2}\|\nabla f(\bar{x}_k)\|^2 - \tfrac{m L^2    }{2}  \textstyle\sum_{i=1}^m \fm{p_{i,r}} \|x_{i,k} - \bar{x}_k\|^2 \notag \\
&= -\tfrac{1}{2}\|\nabla f(\bar{x}_k)\|^2 - \tfrac{m L^2 }{2} \bar{e}_k,
\end{align}
where $\bar{e}_k = \sum_{i=1}^m \fm{p_{i,r}} \|x_{i,k} - \bar{x}_k\|^2$ from Definition \ref{def:main_terms}. Next, we analyze the second term on the right-hand side in \eqref{Main_Term}.    Using the identity
$a^{\top} b \geq -\frac{1}{2\lambda}\|a\|^2 - \frac{\lambda}{2}\|b\|^2$, for $\lambda:=2$, $a := \nabla f(\bar{x}_k)$, and $b := \sum_{i=1}^{m}(p_{i,r} - p_i^*)\nabla f_i(\bar{x}_k)$, we obtain
\begin{align*}
 & \nabla f(\bar{x}_k)^{\top} \textstyle\sum_{i=1}^m (p_{i,r}-p^*_i) \nabla f_i(\bar{x}_k) 
\\ &\geq -\tfrac{1}{4} \|\nabla f(\bar{x}_k) \|^2 -\left\|\textstyle\sum_{i=1}^m (p_{i,r}-p^*_i) \nabla f_i(\bar{x}_k)\right\|^2 \\
&\geq -\tfrac{1}{4} \|\nabla f(\bar{x}_k) \|^2 -m\textstyle\sum_{i=1}^m (p_{i,r}-p^*_i)^2 \| \nabla f_i(\bar{x}_k)\|^2\\
& \geq -\tfrac{1}{4} \|\nabla f(\bar{x}_k) \|^2 - m\fy{\delta_r^2}\textstyle\sum_{i=1}^m \|\nabla f_i(\bar{x}_k)\|^2,
\end{align*}
where we used~\ref{linear_rate}. Invoking Assumption~\ref{assum:bgd}, we obtain
\begin{align} \label{Second_Term}
& \nabla f(\bar{x}_k)^{\top} \textstyle\sum_{i=1}^m (p_{i,r}-p^*_i) \nabla f_i(\bar{x}_k) \notag\\&\geq -  m^2\fy{\delta_r^2} G^2    -\left(\tfrac{1}{4} +  m^2  \fy{\delta_r^2} B^2\right)\|\nabla f(\bar{x}_k)\|^2  .
\end{align}
Taking conditional expectations on both sides of \eqref{Main_Term}, and using \eqref{First_Term} and \eqref{Second_Term}, we obtain the result.
\end{proof}

\begin{lemma}\label{lem:Consensus_Error} 
Consider Algorithm~\ref{alg:IncentFedAvg}. Let Assumptions~\ref{assum:main} and \ref{assum:game} hold.  Then, the following results hold. 

\noindent (i) For any communication round $r \geq 0$ and any iteration $k$ where $T_r +1\leq k \leq T_{r+1}$, \begin{align*} 
\mathbb{E} [\bar{e}_k  ] & \leq \gamma^2 (k-T_r) \textstyle\sum_{t=T_r}^{k-1}\left( 3\nu^2+ 3 L^2\mathbb{E}[\bar{e}_t]  \right.\\ &\left.  +3 \textstyle\sum_{i=1}^m  \mathbb{E}[\|\nabla f_i(\bar{x}_t)\|^2]\right).
\end{align*} 

\noindent \fy{(ii) Let Assumption~\ref{assum:bgd} hold and $\gamma \leq \frac{1}{\sqrt{3} H L}$. Then,} for any communication round $r \geq 0$ and any iteration $k$ where $T_r +1\leq k \leq T_{r+1}$, $$\mathbb{E}[\bar{e}_k] \leq 9H\gamma^2 \textstyle\sum_{t=T_r}^{k-1} (\nu^2 +  m(G^2 + B^2\mathbb{E}[\|\nabla f(\bar{x}_t)\|^2])).$$

\noindent \fy{(iii) Further, under Assumption~\ref{assum:bgd} and $\gamma \leq \frac{1}{\sqrt{3} H L}$,} for any round index $\hat{r}$ such that $0 \leq \hat{r}\leq R-1$ and $R\geq 1$, we have 
\begin{align*}
\textstyle \sum_{k=T_{\hat{r}}}^{T_R-1} \mathbb{E}[\bar{e}_k] & \leq  \textstyle 9m\gamma^2 B^2 H^2 \sum_{k=T_{\hat r}}^{T_R-1}  \mathbb{E}[\|\nabla f(\bar{x}_k)\|^2] \\
&+ 9\gamma^2 (\nu^2 + mG^2) H^2(T_R-T_{\hat{r}}).
\end{align*}
\end{lemma}

\begin{lemma}\label{lemma:gbar_sq}  
Let \fm{Assumptions \ref{assum:main} and \ref{assum:game} hold}. 
Let $\bar{g}_k$ be given by Definition~\ref{def:main_terms}. Then, the following hold. 

\noindent (a) For any $k \geq 1$,  \begin{align*}
    \mathbb{E}[\|\bar{g}_k\|^2] & \leq  \nu^2 + 2mL^2 \mathbb{E}\left[ \bar{e}_k \right] + 4m\fy{\delta_r^2} \textstyle  \sum_{i=1}^{m} \mathbb{E}[\| \nabla f_i(\bar{x}_k) \|^2]\\ & +4 \mathbb{E}\left[\| \nabla f(\bar{x}_k)\|^2\right].
\end{align*}
\noindent (b) Additionally, if Assumption \ref{assum:bgd} holds, then
{\begin{align*}
\mathbb{E}[\|\bar{g}_k\|^2]  &\leq \fm{\nu^2} + 2mL^2 \mathbb{E}\left[ \bar{e}_k \right] + 4m^2\fy{\delta_r^2}G^2  \\
& + 4\left(  m^2\fy{\delta_r^2}B^2 +  1 \right) \mathbb{E}\left[\| \nabla f(\bar{x}_k)\|^2\right].
\end{align*}}
\end{lemma}

{\bf Proof of Theorem~\ref{Theorem:main}.} 
 (i-1) 
 From the $L$-smoothness of the global loss function and Lemma~\ref{lemma:agg}, we may write
\begin{align*}
& f(\bar{x}_{k+1}) \leq f(\bar{x}_k) - \gamma \nabla f(\bar{x}_k)^{\top} \bar{g}_k + \tfrac{L\gamma^2 }{2} \|\bar{g}_k\|^2.
\end{align*}
Taking expectation on the both sides and invoking Lemmas~\ref{lem:agg_gradient} and \ref{lemma:gbar_sq}, we obtain
\begin{align*}
 &(\tfrac{1}{4} -  m^2 \fm{\delta_r^2} B^2 (1+2L\gamma) -2L\gamma)\gamma\mathbb{E}[\|\nabla f(\bar{x}_k)\|^2] \\
 & \leq \mathbb{E}[f(\bar{x}_k)] - \mathbb{E}[f(\bar{x}_{k+1})] + \fm{\tfrac{L \gamma^2  \nu^2}{2}}+ (\tfrac{ 1}{2} +L \gamma ) mL^2 \gamma \mathbb{E}[\bar{e}_k]\\
& + m^2 \fm{\delta_r^2} G^2 \gamma (1+2L\gamma).
\end{align*}
Summing both sides for $k=T_{\hat{r}},\ldots,T_R-1$, dividing by $\gamma (T_R-T_{\hat r})$, and using Lemma~\ref{lem:Consensus_Error}, we obtain
{\begin{align*}
&(\tfrac{1}{4} -  m^2 \fm{\delta_{\hat{r}}^2} B^2 (1+2L\gamma) -2L\gamma) \tfrac{\textstyle\sum_{k=T_{\hat{r}}}^{T_R-1}\mathbb{E}[\|\nabla f(\bar{x}_k)\|^2]}{T_R-T_{\hat r}} \\
&\leq (\gamma (T_R-T_{\hat r}))^{-1}(\mathbb{E}[f(\bar{x}_{T_{\hat{r}}})] - \mathbb{E}[f(\bar{x}_{T_R})])+ \fm{\tfrac{L \gamma  \nu^2}{2}\textstyle} \\
& + (\tfrac{ 1}{2} +L \gamma ) mL^2 (9\gamma^2 (\nu^2 + mG^2) H^2) \\
& + (\tfrac{ 1}{2} +L \gamma ) mL^2 (\tfrac{9m\gamma^2 B^2 H^2}{T_R-T_{\hat r}}\textstyle\sum_{k=T_{\hat{r}}}^{T_R-1} \mathbb{E}[\|\nabla f(\bar{x}_k)\|^2])\\
&+  m^2 \fy{\delta_{0}}^2  G^2  (1+2L\gamma)  \tfrac{1}{T_R-T_{\hat r}}\textstyle\sum_{k=T_{\hat{r}}}^{T_R-1}\rho^{2r}.
\end{align*}} 
Recall that $T_r := H\, r$ for any $r\geq 0$, implying that $T_R-T_{\hat r} = H(R-\hat{r})$. We  may write
\begin{align*}
\tfrac{1}{T_R-T_{\hat r}}\textstyle\sum_{k=T_{\hat{r}}}^{T_R-1}\rho^{2r} & \leq \tfrac{1}{T_R-T_{\hat r}}\textstyle\sum_{r=\hat{r}}^{R-1}H \rho^{2r}  
= \tfrac{\rho^{2\hat{r}}-\rho^{2R}}{(1-\rho^2)(R-\hat{r})}.
\end{align*}
In view of $\gamma \leq  \sqrt{2}/{\sqrt{153} m B H  L}$ and $\gamma L \leq \tfrac{1}{32}$, we have $(\tfrac{ 1}{2} +L \gamma ) m L^2  9m\gamma^2 B^2 H^2 \leq   \tfrac{1}{16} .$ Further,  the assumption \fm{$\hat{r}\geq \tfrac{\ln(17m^2\delta_{0}^2  B^2)}{2\ln(1/\rho)}$}, implies that $m^2 \fm{\delta_{\hat{r}}^2} B^2 (1+2L\gamma) \leq \tfrac{1}{16}.$ From the preceding inequalities, we obtain
\begin{align*}
&  (\tfrac{1}{16})  \tfrac{1}{T_R-T_{\hat r}}\textstyle\sum_{k=T_{\hat{r}}}^{T_R-1}\mathbb{E}[\|\nabla f(\bar{x}_k)\|^2]  
 \\
& \leq (\gamma (T_R-T_{\hat r}))^{-1}(\mathbb{E}[f(\bar{x}_{\hat{r}})] - f^*)+ \tfrac{L \gamma  \nu^2}{2}\\
&+    m^2 \fy{\delta_{0}}^2  G^2 (1+2L\gamma )  \tfrac{\rho^{2\hat{r}}-\rho^{2R}}{(1-\rho^2)(R-\hat{r})} \\
&  +  \tfrac{153}{32} (\nu^2 + mG^2) m L^2 H^2\gamma^2.
\end{align*}
Invoking the definition of $k^*$, we obtain the result in (i). 

\noindent (i-2)
 {Note that since $k^*$ is chosen uniformly at random between $T_{\hat{r}}$ and $T_{R}-1$, and that the number of local steps in each round is constant and is equal to $H$,  we have that $r^*$ is uniformly distributed in $\{\hat{r},\ldots, R-1\}$.  We also have that $r^*$ should satisfy $T_{r^*} \leq k^* \leq T_{r^*+1}-1$. Invoking Lemma~\ref{linear_rate}, we have $\|\hat{N}_{r^*}-N^*\|^2 \leq \|\hat{N}_0-N^*\|^2 \rho^{2r^*}$. Thus, $\mathbb{E}[\|\hat{N}_{r^*}-N^*\|^2] \leq \|\hat{N}_0-N^*\|^2 \mathbb{E}[\rho^{2r^*}]$. We have $\mathbb{E}[\rho^{2r^*}] = \frac{1}{R-\hat{r}}\sum_{r^*=\hat{r}}^{R-1} \rho^{2r^*}=\frac{\rho^{2\hat{r}} - \rho^{2R}}{(R - \hat{r} )(1 - \rho^2)}$.  }
 
\noindent (ii) Consider the bound in (i-1). To ensure that this inequality holds, it is necessary to have $\gamma \leq (\min\left\{\frac{1}{32}, \tfrac{\sqrt{2}}{\sqrt{153} m B H  }\right\})\frac{1}{L}$, $   \fm{\tfrac{\ln(17m^2\delta_{0}^2  B^2)}{2\ln(1/\rho)}} \leq  \hat{r}$. We have $\gamma := \sqrt{\frac{1}{R_{\varepsilon} H}}$, $H$ as a constant, $T_R:= R_\varepsilon H$, and $T_{\hat{r}}:=\frac{R_\varepsilon H}{2}$. Further $\hat{r} = T_{\hat{r}}/H = R_\varepsilon/2$. From $R_\varepsilon \geq R_3$ and $R_\varepsilon H \geq M\max\{(m BLH)^2,L^2\}$, for some suitable $M>0$, the two aforementioned conditions are satisfied, and thus, the bound in (i-1) holds. We obtain
\begin{align*}
\mathbb{E}[\|\nabla f(\bar{x}_{k^*})\|^2] &\leq 
32 (\hat{f} - f^*)\tfrac{1}{\sqrt{R_\varepsilon H}}+  \tfrac{8 L   \nu^2 }{\sqrt{R_\varepsilon H}} \\
& +  \tfrac{153}{2} (\nu^2 + mG^2) m L^2 H^2 \tfrac{1}{R_\varepsilon H} \\
&+   32 m^2 \fy{\delta_{0}}^2  G^2 (1+\tfrac{2L}{\sqrt{R_\varepsilon H}} )  \tfrac{\rho^{2\hat{r}}}{(1-\rho^2)R_\varepsilon H}\\
&\leq 
153\sqrt{\tfrac{(\hat{f} - f^*)^2+L^2\nu^4+m^2L^4H^4(\nu^2+mG^2)^2}{R_\varepsilon H}}   \\
& +     ( \tfrac{96 m^2 \fy{\delta_{0}}^2  G^2}{(1-\rho^2) L^2})\rho^{R_\varepsilon}\fy{,}
\end{align*}
where we used $2\hat{r} = R_\varepsilon$ and assumed that $R_\varepsilon H \geq L^2 $. Assuming that $R_\varepsilon \geq (306)^2 R_{1,\varepsilon}$, we obtain 
\begin{align*}
\mathbb{E}[\|\nabla f(\bar{x}_{k^*})\|^2] & \leq 
\tfrac{\varepsilon}{2}  +     ( \tfrac{96 m^2 \fy{\delta_{0}}^2  G^2}{(1-\rho^2) L^2})\rho^{R_\varepsilon}.
\end{align*}
From $R_\varepsilon \geq R_{2,\varepsilon},$ 
we have that 
$ ( \tfrac{96 m^2 \fy{\delta_{0}}^2  G^2}{(1-\rho^2) L^2})\rho^{R_\varepsilon} \leq \tfrac{\varepsilon}{2}$. Thus, we obtain $\mathbb{E}[\|\nabla f(\bar{x}_{k^*})\|^2] \leq \varepsilon$. It suffices to show that  $\ \mathbb{E}[\|\hat{N}_{r^*}-N^*\|^2]  \leq \varepsilon$. From (i-2),  substituting $2\hat{r} = R_\varepsilon$ and $R-\hat{r} = R_\varepsilon/2$, and using $R_\varepsilon H \geq L^2$, we have 
\begin{align*} &\mathbb{E}[\|\hat{N}_{r^*}-N^*\|^2] \leq  \tfrac{\|\hat{N}_0 - N^*\|^2(\rho^{2\hat r}-\rho^{2R})}{(1-\rho^2)(R-\hat{r})} \leq \tfrac{2\|\hat{N}_0 - N^*\|^2\rho^{R_\varepsilon}}{(1-\rho^2)R_\varepsilon}\\
&\leq  \tfrac{\|\hat{N}_0 - N^*\|^2 \max\{\sqrt{192} m  (m-1)\|N^*\|_1^{-1} G,1\}^2\rho^{R_\varepsilon}}{(1-\rho^2)L^2}\leq \varepsilon  ,\end{align*}
where the last inequality follows from $R_\varepsilon \geq R_{2,\varepsilon}$.

\noindent {\bf Proof of Theorem~\ref{thm:nonconvex_mm}.}

\noindent \fr{(i) The proof follows the proof of
Theorem~\ref{Theorem:main} (i-1) up to the per-iteration
drift inequality, since Lemmas~\ref{lem:agg_gradient} and
\ref{lemma:gbar_sq} hold under $|p_{i,r} - p_i^*| \leq \hat{\delta}_r$
 and do not require strong monotonicity. From the
$L$-smoothness of $f$, Lemma~\ref{lemma:agg}, and
Lemmas~\ref{lem:agg_gradient} and \ref{lemma:gbar_sq}, we obtain
\begin{align*}
&(\tfrac{1}{4} - m^2 \delta_r^2 B^2(1+2L\gamma) - 2L\gamma)\,\gamma\,
\mathbb{E}[\|\nabla f(\bar{x}_k)\|^2]\\
&\le \mathbb{E}[f(\bar{x}_k)] - \mathbb{E}[f(\bar{x}_{k+1})]
+ \tfrac{L\gamma^2 \nu^2}{2}
+ (\tfrac{1}{2}+L\gamma)\,m L^2 \gamma\,\mathbb{E}[\bar{e}_k]\\
& + m^2 \delta_r^2 G^2 \gamma(1+2L\gamma).
\end{align*}
}
\fyrr{\noindent\textit{Departure from the strongly monotone proof.} At
this point, the proof of Theorem~\ref{Theorem:main} substitutes
$\delta_r = \delta_0 \rho^r$ from Lemma~\ref{linear_rate} (ii), which
delivers a geometric contraction factor $\rho^{2\hat r}$. Under
Assumption~\ref{assum:game_monotone}, geometric contraction is no
longer available. By Lemma~\ref{lem:welfare} (iii) and strict
convexity of $h$ on $\mathcal{N}^* \subseteq \mathcal{N}$, we  have
$\bar N_R \to N^*$. Invoking the assumption that $\hat{N}_r$ has a limit point, then by Ces\`aro mean theorem $\hat{N}_r \to N^*$. Consequently, $\hat{\delta}_r$ remains a generic
vanishing sequence rather than a geometric one, and we retain it
symbolically in the analysis.}

\fyrr{Notably, in view of Lemma~\ref{lem:welfare} (iv), $\hat{\delta}_r \to 0$, and thus there exists $\hat{r}$ such that
$17 m^2 \hat{\delta}_{\hat{r}}^2 B^2 \le 1$. Combined with
$\gamma L \le 1/32$, this gives
$m^2 \hat{\delta}_r^2 B^2(1+2L\gamma) + 2L\gamma \le 1/16 + 1/16 = 1/8$, so
$\tfrac{1}{4} - m^2 \hat{\delta}_r^2 B^2(1+2L\gamma) - 2L\gamma \ge 1/8$
uniformly in $r$. We obtain \begin{align*}
&  (\tfrac{1}{16})  \tfrac{1}{T_R-T_{\hat r}}\textstyle\sum_{k=T_{\hat{r}}}^{T_R-1}\mathbb{E}[\|\nabla f(\bar{x}_k)\|^2]  
 \\
& \leq (\gamma (T_R-T_{\hat r}))^{-1}(\mathbb{E}[f(\bar{x}_{\hat{r}})] - f^*)+ \tfrac{L \gamma  \nu^2}{2}\\
&+    m^2    G^2 (1+2L\gamma )  \tfrac{1}{T_R-T_{\hat r}}\textstyle\sum_{k=T_{\hat{r}}}^{T_R-1} \hat{\delta}_{r(k)}^2 \\
&  +  \tfrac{153}{32} (\nu^2 + mG^2) m L^2 H^2\gamma^2.
\end{align*}
Multiplying the both sides by $16$ and invoking the definition of $k^*$, we obtain the bound.} 

\noindent \fr{(ii) Taking
$\limsup_{R\to\infty}$ on the both sides in (i), the first term on the right vanishes, and the third term vanishes since
$\hat{\delta}_r \to 0$ as $R \to \infty$ implies
$\tfrac{1}{R}\sum_{r=\hat{r}}^{R-1}\hat{\delta}_r^2 \to 0$ by the Ces\`aro mean theorem. By
the choice of $\gamma$, the second term satisfies $8L\gamma\nu^2 \le
\varepsilon/2$ and the third term satisfies
$\tfrac{153}{2}(\nu^2+mG^2)\,mL^2 H^2 \gamma^2 \le \varepsilon/2$,
and hence $\limsup_{R\to\infty}\mathbb{E}[\|\nabla f(\bar{x}_{k^*})\|^2]
\le \varepsilon$.}

\section{Numerical Results}

We implement \texttt{IncentFedAvg} on both MNIST and CIFAR-10 datasets and demonstrate the practical effectiveness of the method in each case. 

\subsection{\fr{Strongly monotone game}} For the MNIST experiment, we utilize a two-layer neural network with input dimension $d = 784$, hidden layer size $M$, and $C = 10$ output classes. The weight matrices $\mathbf{X}^{(1)} \in \mathbb{R}^{M \times (d+1)}$ and $\mathbf{X}^{(2)} \in \mathbb{R}^{C \times (M+1)}$ are initialized using scaling factors optimized for the chosen learning rate $\gamma$. For the CIFAR-10 experiment, we adopt a convolutional neural network architecture suited for the $32 \times 32 \times 3$ input images, with two convolutional layers followed by fully connected layers and $C = 10$ output classes. For CIFAR-10, the $32 \times 32 \times 3$ images are flattened to $d = 3072$, resulting in $\mathbf{X}^{(1)} \in \mathbb{R}^{3073 \times 128}$ and $\mathbf{X}^{(2)} \in \mathbb{R}^{129 \times 10}$. Both weight matrices are initialized with scaled random values, and stepsize $\gamma = 0.005$.

The local loss for each client $i$ is defined by the cross-entropy $E = -\tfrac{1}{J}\textstyle\sum_{j=1}^{J} \sum_{c=1}^{C} v_{jc} \log(v^{\prime}_{jc})$, where $v^{\prime}_{jc}$ is the predicted probability derived from the softmax output of the pre-activation $a_c^{(2)}$. 
 
We employ the random discovery model for payoff functions $a_i(N) = q_i Q^{\top}N = \sum_{j=1}^m p_ip_j^{\top}N_j$ and linear model for cost functions $c_i(N_i)= \theta_i N_i$ (cf. Section~\ref{sec:prob}). 
$\theta_i$ is randomly chosen in $[0,\|q_i\|^2]$, and we regularize the agents' net utility functions to ensure strong monotonicity of the Nash game, with regularization parameter $\lambda:=10^{-5}$.

\begin{figure}[t]
\centering
\includegraphics[width=0.48\linewidth]{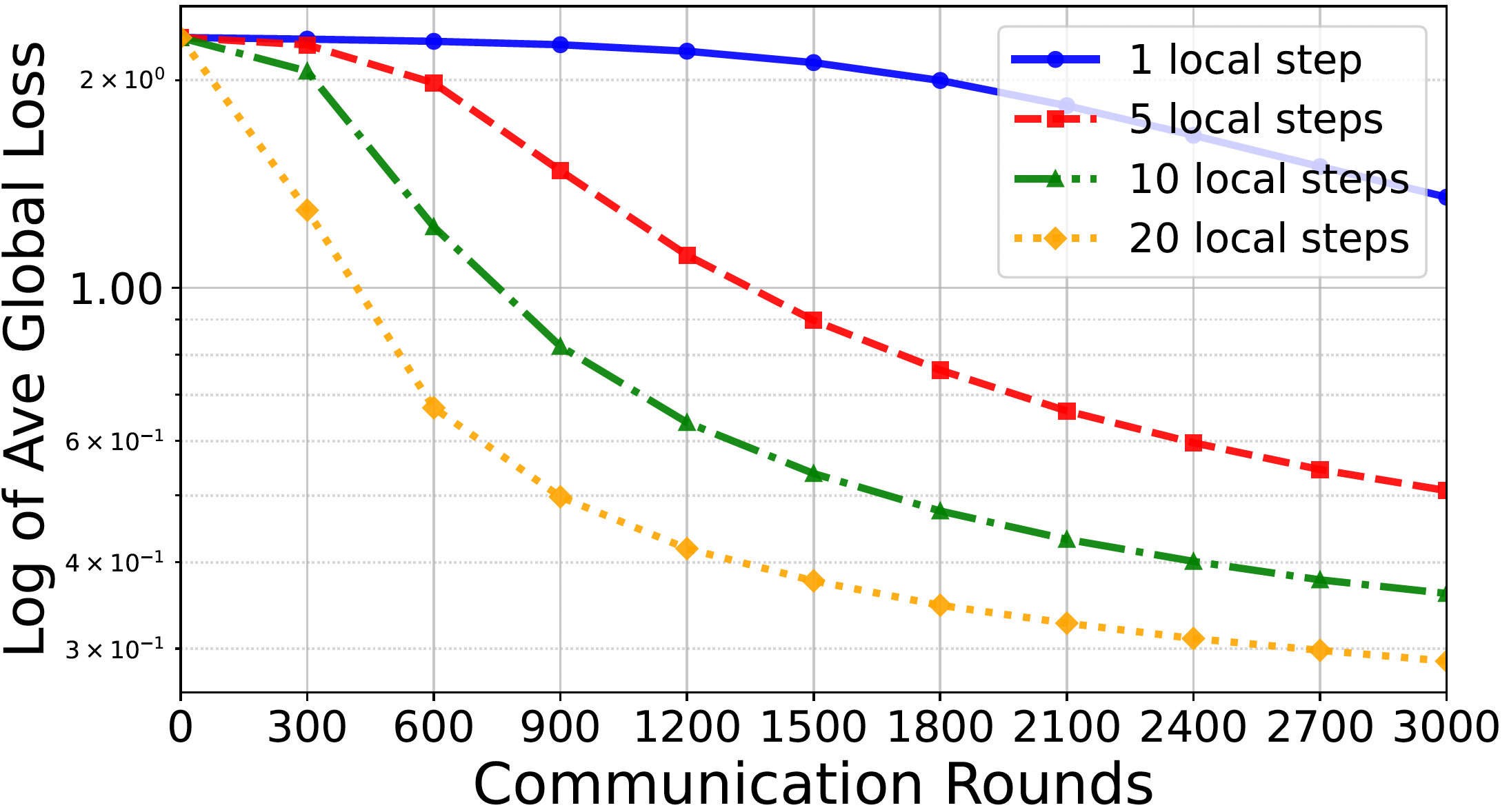}
\includegraphics[width=0.48\linewidth]{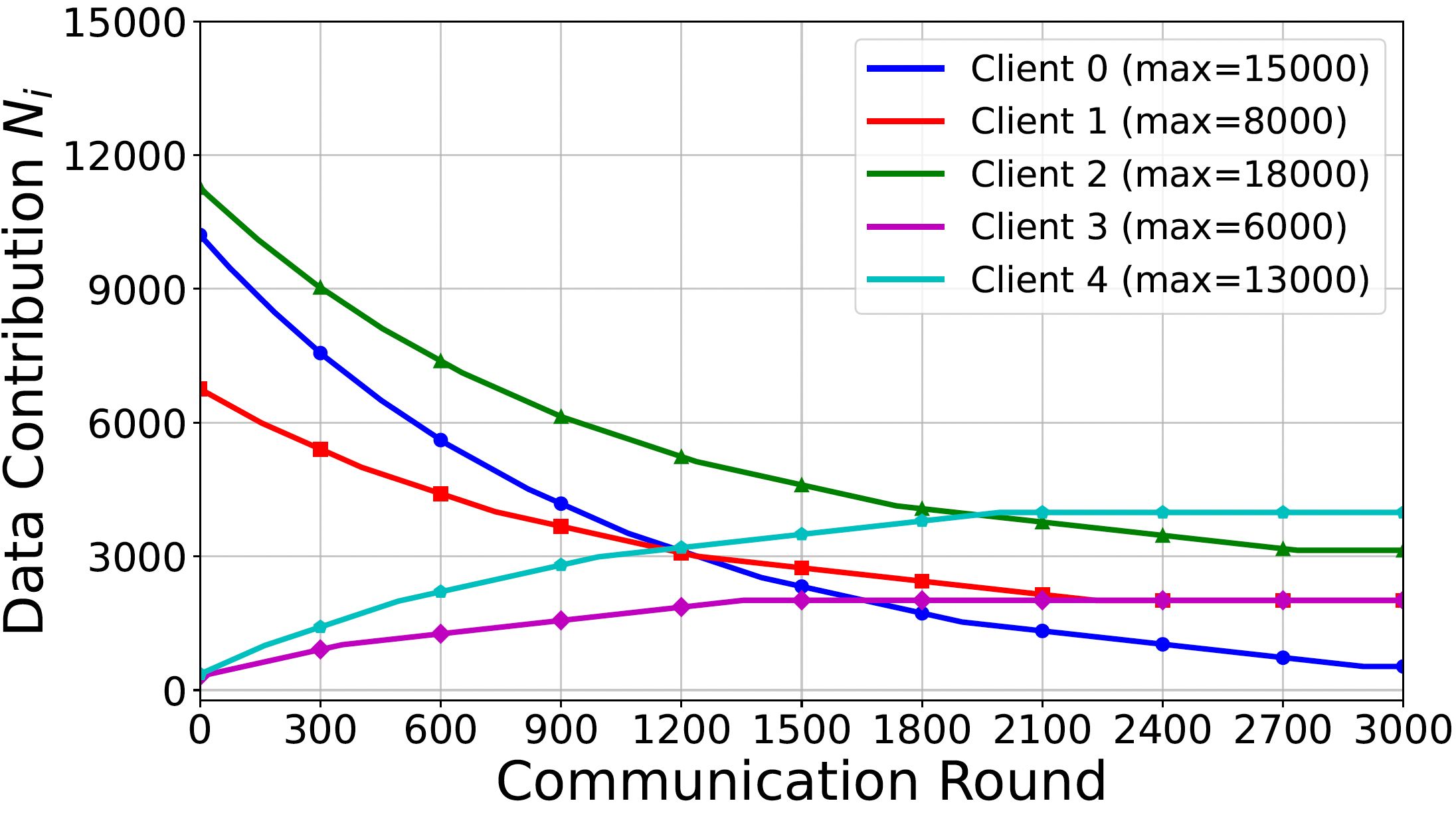}
\caption{\textbf{MNIST.} Left: Global cross-entropy loss across different local steps $H$. Right: Convergence of client contributions $N_{i,r}$ toward the Nash equilibrium under \texttt{IncentFedAvg}.}
\label{fig:mnist_results}
\end{figure}

\begin{figure}[t]
\centering
\includegraphics[width=0.48\linewidth]{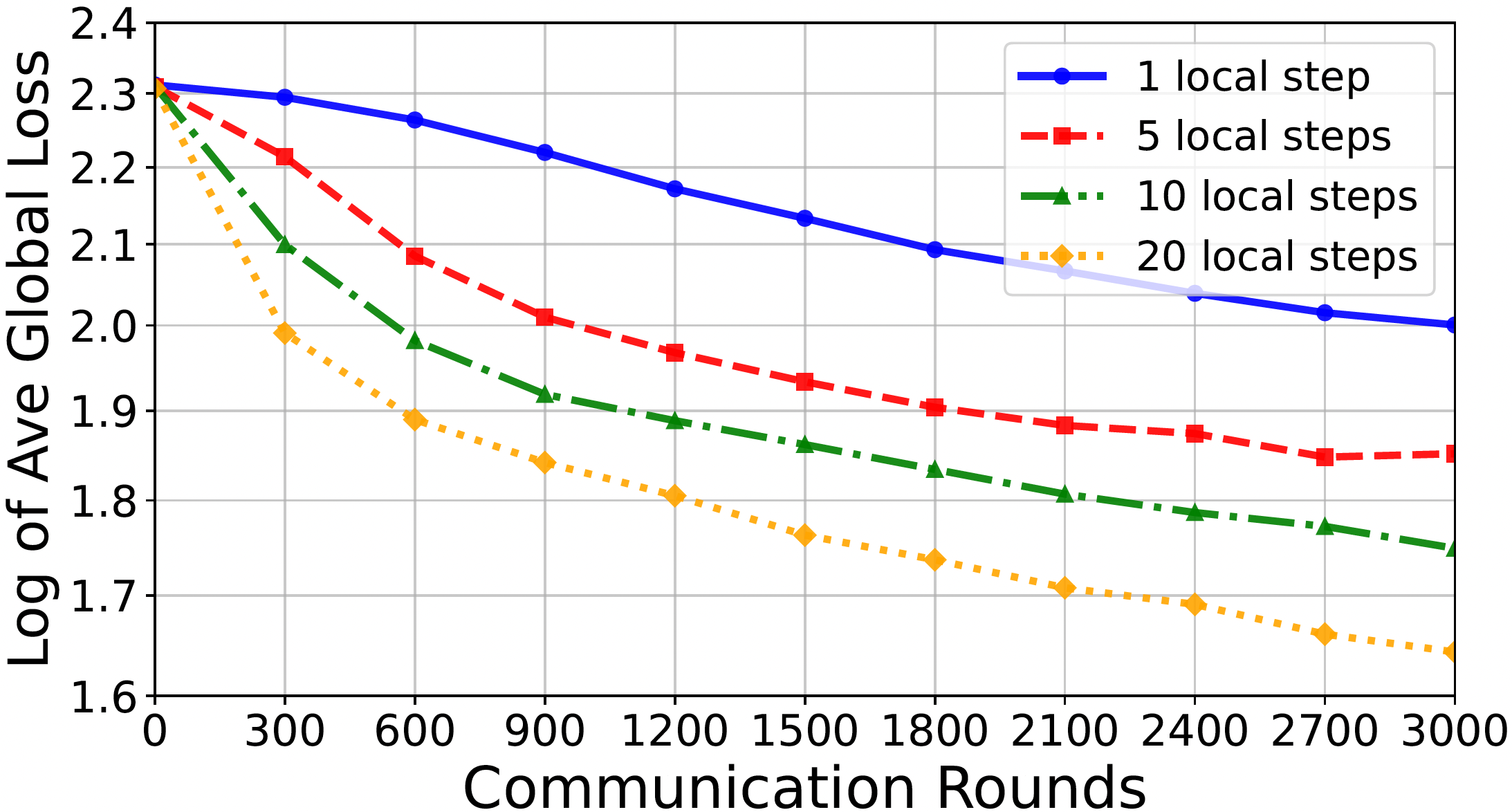}
\includegraphics[width=0.48\linewidth]{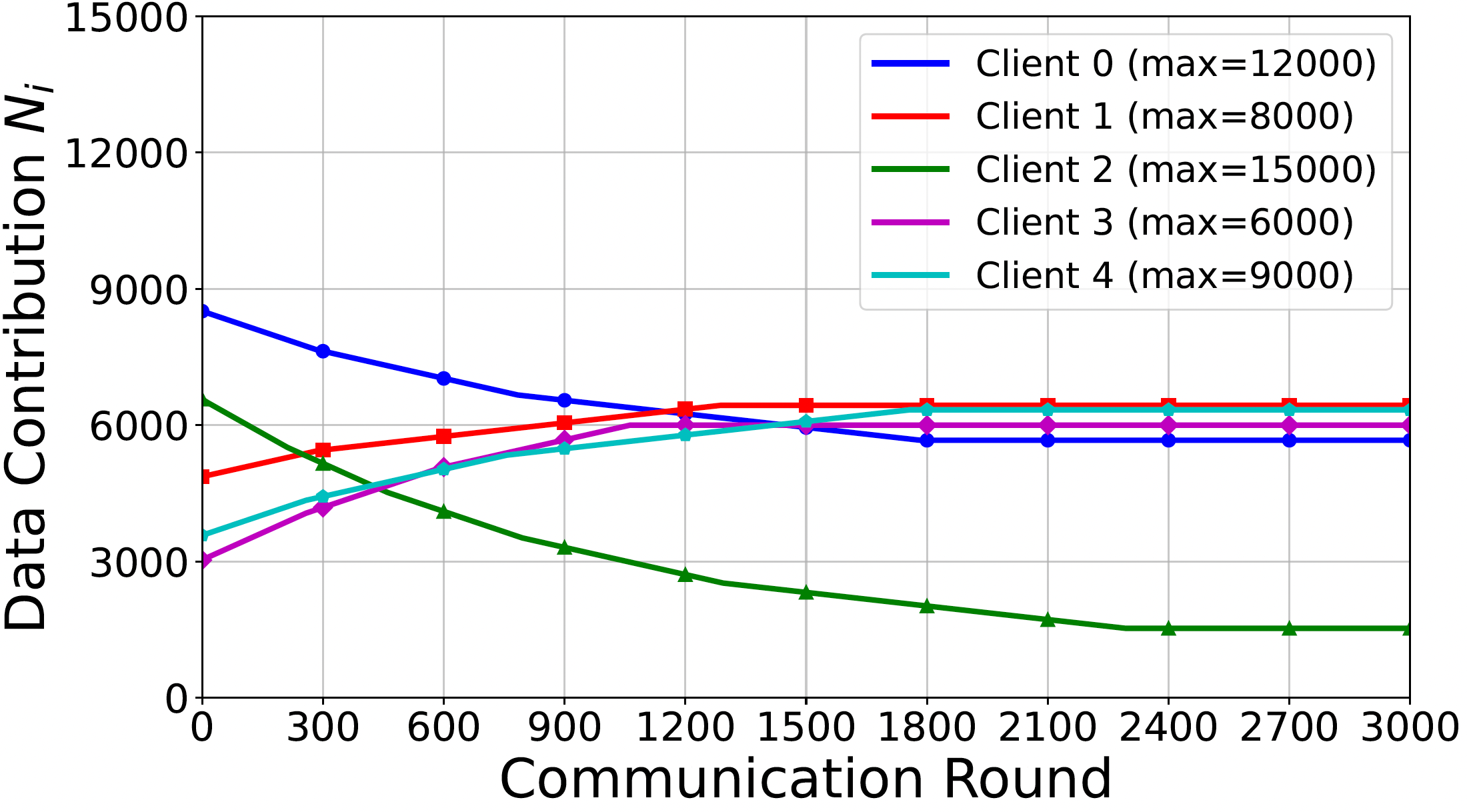}
\caption{\textbf{CIFAR-10.} Left: Global cross-entropy loss across different local steps $H$. Right: Convergence of client contributions $N_{i,r}$ toward the Nash equilibrium under \texttt{IncentFedAvg}.}
\label{fig:cifar_results}
\end{figure}

\subsubsection{Observations and insights}
Fig.~\ref{fig:mnist_results} and Fig.~\ref{fig:cifar_results} present the numerical results. The left plots show the global cross-entropy loss versus communication rounds for different local update steps $H \in \{1,5,10,20\}$, while the right plots show the evolution of client data contributions $N_{i,r}$, illustrating convergence to the  equilibrium of the participation game.

\paragraph{Impact of local computation ($H$)}
The loss curves in Fig.~\ref{fig:mnist_results} and Fig.~\ref{fig:cifar_results} show that larger $H$ indeed accelerates convergence. In particular, $H=1$ converges the slowest, while $H=20$ reaches low loss values in fewer communication rounds. This reflects a standard federated learning trade-off: increased local computation allows clients to make more progress before synchronization, improving communication efficiency. The effect is particularly noticeable on CIFAR-10 despite its higher task complexity.

\paragraph{Nash equilibrium stability}
The right plots in Fig.~\ref{fig:mnist_results} and Fig.~\ref{fig:cifar_results} show that client contributions $N_{i,r}$ converge to stable values, indicating the emergence of a sustainable Nash equilibrium. Clients with heterogeneous utilities settle at their contribution levels across both datasets. Interestingly, in both settings, there exist clients whose participation strategy has either decreased or increased in reaching stability, highlighting the trade-off between the payoff and cost function of each client. 

\subsection{\fr{Merely monotone game: welfare-selected equilibrium}}
\label{subsec:numerics-mm}

\fr{We next evaluate \texttt{IncentFedAvg} on the merely monotone formulation in~\eqref{prob:incentive_FL_welfare}, where the participation game admits a set $\mathcal{N}^{*}$ of equilibria and the welfare-selection objective $h$ identifies a unique target $N^{*} \in \arg\min_{N \in \mathcal{N}^{*}} h(N)$. We keep the random discovery payoff $a_i(N) = q_i Q^{\top} N$ and the linear cost $c_i(N_i) = \theta_i N_i$, and drop the strong-monotonicity regularizer used in the first experiment, so that $F$ is merely monotone. For the welfare loss we use the soft-plus \fyrr{function}}
\fr{$
  h(N) \;=\; \log\!\Big(1 + \exp\!\big(\textstyle\sum_{i=1}^{m} N_i\big)\Big),$}
\fr{which is strictly convex, and by rewarding large aggregate participation it selects the welfare-improving equilibrium in $\mathcal{N}^{*}$. The participation update uses the rate conditions of Lemma~\ref{lem:welfare} with $\tilde{\gamma}_r = \tilde{\gamma}_0 (r+1)^{-a}$ and $\lambda_r = \lambda_0 (r+1)^{-b}$, where we set $a = 0.5$ and $b = 0.25$. We take $\mathcal{N} = \prod_{i=1}^{m}[N_i^{\min}, N_i^{\max}]$ with $m = 5$ and per-client upper bounds $(N_i^{\max})_{i=0}^{4} = (15000, 8000, 18000, 6000, 13000)$.}

\fr{Fig.~\ref{fig:mnist_mm} (left) reports the per-client trajectories $N_{i,r}$ on MNIST. Three clients (Clients $0,2,4$) converge to their upper bounds $N_i^{\max}$ within roughly $1000$ rounds, while Clients $1,3$ are driven to zero. Clients whose marginal discovery payoff dominates their linear cost reach their maximum participation, and clients for which $\theta_i$ dominates exit the game. The new effect, absent in the strongly monotone setting, is that the iteratively regularized update reliably selects the same welfare-maximizing equilibrium across runs. Fig.~\ref{fig:mnist_mm} (right) shows the aggregate $\sum_i N_{i,r}$, which after a short transient stabilizes at $46{,}000 = N_0^{\max} + N_2^{\max} + N_4^{\max}$, matching the welfare-selected equilibrium.}

\fr{\begin{figure}[t]
\centering
\includegraphics[width=0.48\linewidth]{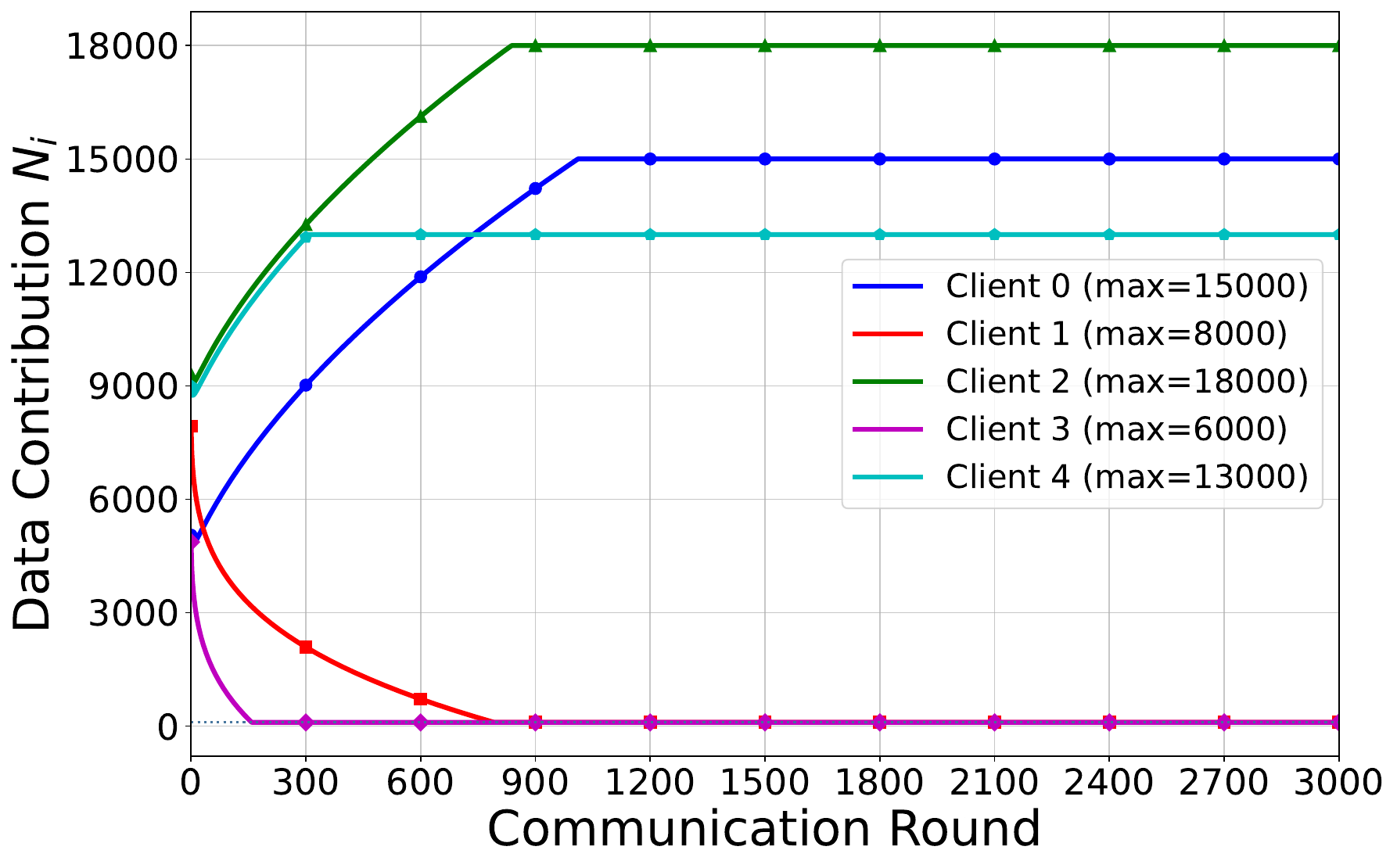}
\includegraphics[width=0.48\linewidth]{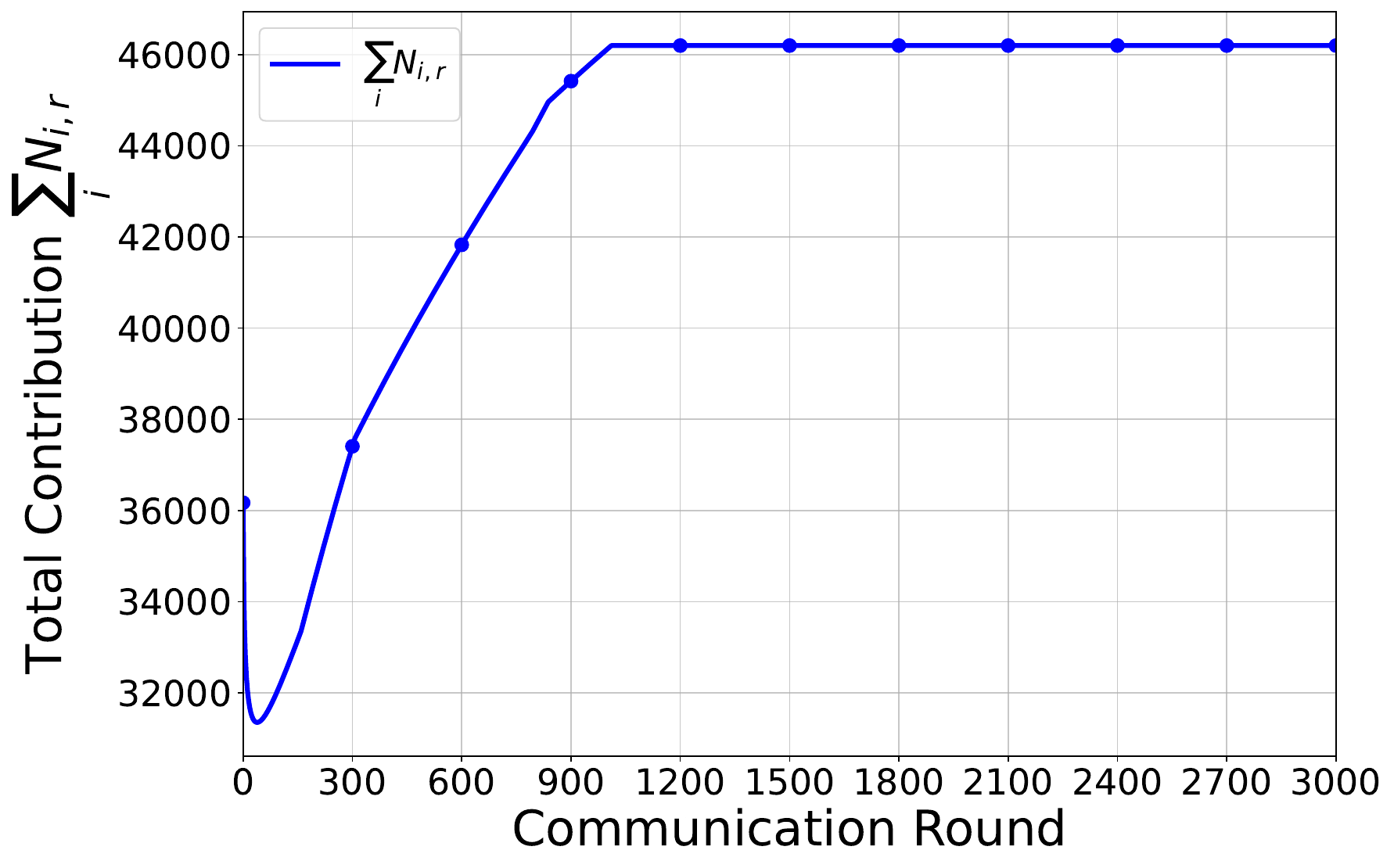}
\caption{\textbf{MNIST, merely monotone game.} Left: Per-client contributions $N_{i,r}$ under the iteratively regularized update of Algorithm~\ref{alg:IncentFedAvg} with welfare loss $h(N)=\log(1+\exp(\sum_i N_i))$, random discovery payoff, and linear cost. Three clients converge to their maximum participation level $N_i^{\max}$ while two are driven to the minimum, identifying the welfare-selected equilibrium. Right: Aggregate contribution $\sum_i N_{i,r}$ converges to $46{,}000$, the sum of $N_i^{\max}$ over the surviving clients.}
\label{fig:mnist_mm}
\end{figure}}

\fr{\subsection{Zero-sum game with power-law payoff}
\label{subsec:numerics-zs}

To stress-test the algorithm in a setting where contributions are unambiguously competitive, we consider a payoff--cost configuration in which one client's reward is funded by the contributions of the others.}

\fr{\begin{definition}[Zero-sum participation game]
\label{def:zero-sum}
A data participation game is \emph{zero-sum} if each client's cost equals the sum of the other clients' payoffs,
\begin{align}
c_i(N) \;=\; \textstyle\sum_{j\neq i} a_j(N), \qquad i \in [m],
\end{align}
so that the per-client utility loss is $l_i(N) = \sum_{j \neq i} a_j(N) - a_i(N)$. For $m=2$ this is a zero-sum game in the standard sense ($l_1 + l_2 = 0$); for $m \geq 3$ the configuration retains the property that any unit of payoff to one client is registered as cost by every other client.
\end{definition}}

\fr{\paragraph{Power-law payoff} We adopt the empirical scaling-law form of Kaplan~et~al.~\cite{kaplan2020scaling}, in which the cross-entropy loss on a neural model scales with dataset size $m$ as $\ell(m) = \alpha \cdot m^{-\beta}$ for some $\alpha > 0$ and $\beta \in (0,1]$. The resulting task accuracy, used as the client payoff, is}
\fr{\begin{align}
  a_i(N) \;=\; 1 - \alpha_i \, \|N\|_1^{-\beta_i}, \qquad \alpha_i > 0, \;\; \beta_i \in (0,1]
  \label{eq:power-law-payoff}
\end{align}}
\fr{This payoff is nonnegative and nondecreasing in the aggregate contribution $\|N\|_1 = \sum_j N_j$, and matches the scaling behavior observed in large neural models~\cite{kaplan2020scaling,henighan2020scaling}, where more data improves performance but each additional data point helps less than the previous one. We draw $\alpha_i$ and $\beta_i$ per client to model heterogeneous scaling, and keep the same upper bounds $(N_i^{\max})_{i=0}^{4} = (15000, 8000, 18000, 6000, 13000)$ as in Section~\ref{subsec:numerics-mm} for direct comparability.}

\fr{\paragraph{Participation dynamics} Under~\eqref{eq:power-law-payoff} and Definition~\ref{def:zero-sum}, the partial derivative of client $i$'s utility loss is
\begin{align*}
  \nabla_{N_i} l_i(N)
  \;=\; \textstyle\sum_{j\neq i} \alpha_j \beta_j \|N\|_1^{-\beta_j - 1} \;-\; \alpha_i \beta_i \|N\|_1^{-\beta_i - 1}.
\end{align*}}
\fr{The sign of this expression partitions clients into two groups. Clients whose own scaling coefficients dominate the sum of the others' decrease $l_i$ by raising $N_i$ and reach their maximum capacity, while the rest are pushed to the lower bound of the strategy set.}

\fr{Fig.~\ref{fig:mnist_zs} reports the loss curves and per-client trajectories on MNIST. The partition predicted by the sign analysis is visible in the right plot. Clients $0$ and $4$, whose scaling coefficients dominate, climb to their upper bounds $N_i^{\max}$, while Clients $1,2,3$ decay approximately linearly under the projection and exit at zero by round $\approx 1400$. The decay is monotone rather than oscillatory, since once a client's projected gradient turns inward, the boundary becomes absorbing under the projected update. 
}

\fr{\begin{figure}[t]
\centering
\includegraphics[width=0.48\linewidth]{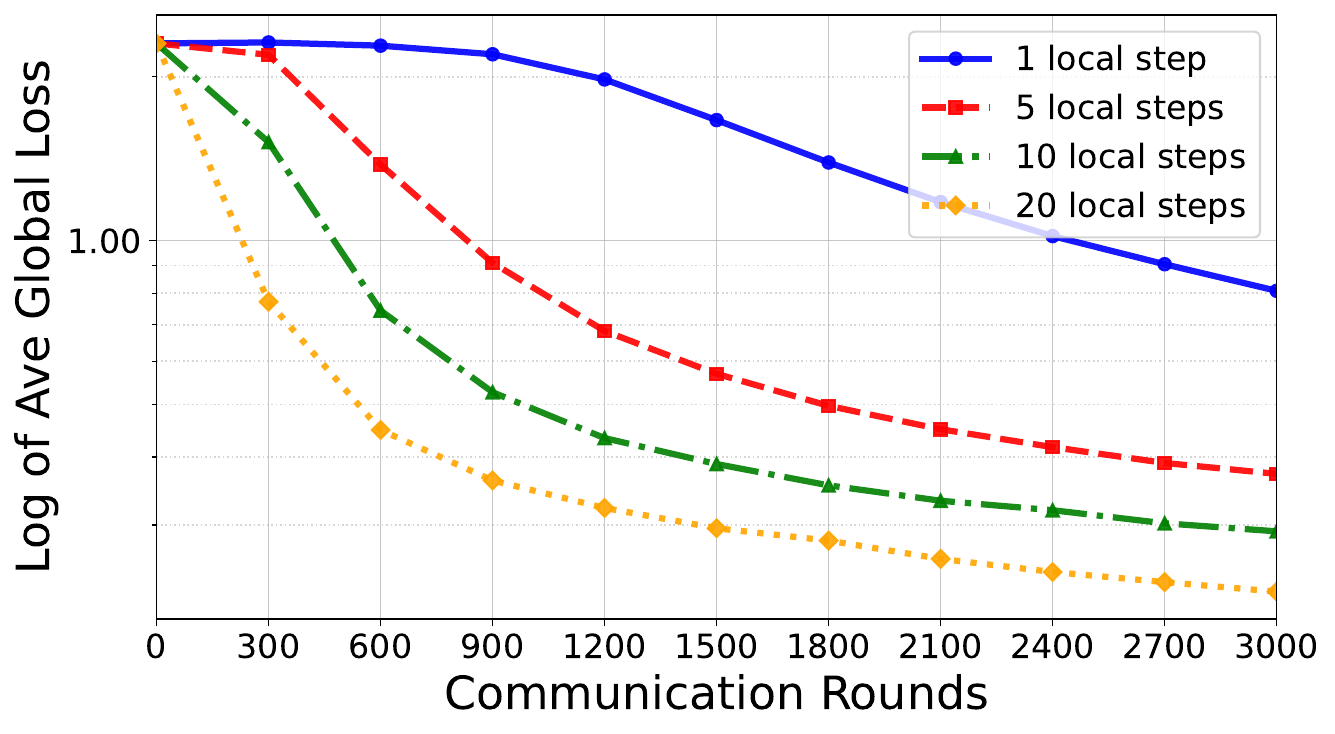}
\includegraphics[width=0.48\linewidth]{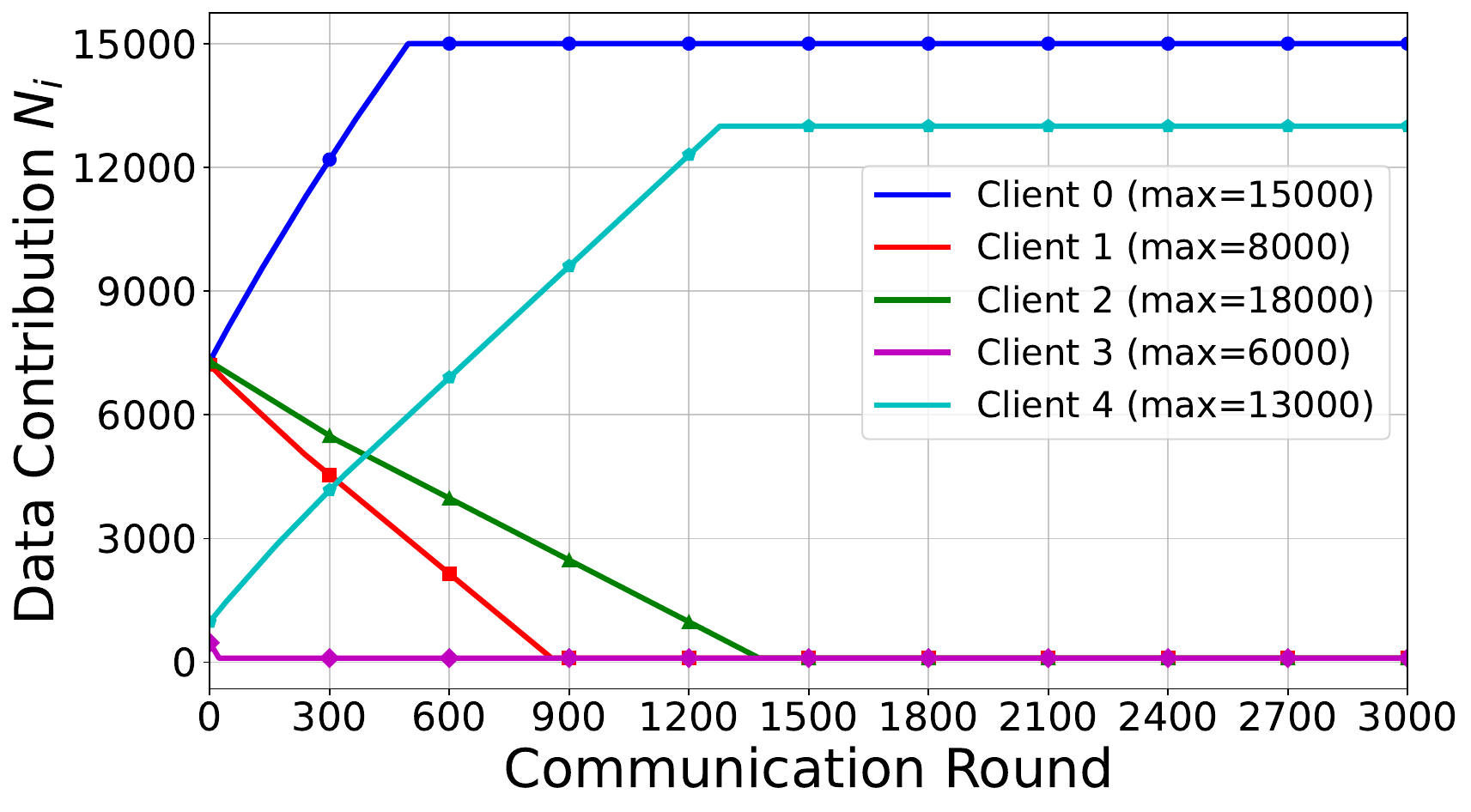}
\caption{\textbf{MNIST, zero-sum game with power-law payoff.} Left: Global cross-entropy loss across different local steps $H$. Right: Per-client contributions $N_{i,r}$ under the Kaplan~et~al.~scaling-law payoff $a_i(N) = 1 - \alpha_i \|N\|_1^{-\beta_i}$ and cost $c_i(N) = \sum_{j\neq i} a_j(N)$ from Definition~\ref{def:zero-sum}. Clients with the dominant scaling coefficients (Clients $0,4$) converge to their upper bounds $N_i^{\max}$, while the remaining clients exit the game.}
\label{fig:mnist_zs}
\end{figure}}

\fr{\paragraph{Comparison across game configurations} Across the three configurations, strongly monotone discovery (Figs.~\ref{fig:mnist_results},~\ref{fig:cifar_results}), merely monotone discovery with welfare selection (Fig.~\ref{fig:mnist_mm}), and zero-sum with power-law payoff (Fig.~\ref{fig:mnist_zs}), \texttt{IncentFedAvg} converges to a stable participation profile, with the limit determined by the game's incentive structure rather than by initialization. The strongly monotone discovery game returns a heterogeneous interior equilibrium, welfare selection in the merely monotone game pushes total participation to the sum of the high-capacity clients' upper bounds, and the zero-sum power-law game concentrates participation in the clients with the dominant scaling coefficients. In all three, the participation equilibrium emerges within an order of magnitude of the rounds needed for the global model to converge, confirming that the coupled cooperative and noncooperative design produces compatible time scales.}

\section{Conclusions}
We proposed \texttt{IncentFedAvg}, an incentive-aware FedAvg method that incorporates strategic data participation into federated learning. The approach couples cooperative local model training with a noncooperative Nash game determining client data contributions, allowing clients to adjust participation based on payoff–cost tradeoffs.  \fyrr{For strongly monotone games, we establish performance guarantees for convex and nonconvex objectives. For merely monotone games, we prove asymptotic convergence under welfare loss minimization.} \fr{While the proposed framework captures strategic data participation 
through dataset size $N_i$, it does not explicitly model data 
heterogeneity at the level of individual sample quality or local 
distribution shift. Notably, the random discovery payoff 
$a_i(N) = q_i Q^\top N$ does encode each client's class 
distribution $q_i$, so clients with more informative or 
complementary distributions receive higher marginal payoffs. 
However, the strategic variable $N_i$ remains a scalar quantity 
representing dataset size, and differences in contribution 
value arising from distributional heterogeneity are not 
explicitly captured in the Nash game formulation. Extending the 
framework to incorporate distribution-aware contribution metrics represents an important and practically relevant direction for 
future work.}




\bibliographystyle{IEEEtran}
\bibliography{ref_v4}

\newpage
\section*{APPENDIX}

\subsection{Supplementary material}
\begin{remark}
If Assumption~\ref{assum:main} (i) holds, then the global loss function $f$ is $L$-smooth. This is because for any $x,y \in \mathbb{R}^n$ we have
\begin{align*}
&\|\nabla f(x) - \nabla f(y)\| = \left\|\textstyle \sum_{i=1}^m p_i^* (\nabla f_i(x) -   \nabla f_i(y)\right\| \\
&\leq \textstyle\sum_{i=1}^m\left\| p_i^* (\nabla f_i(x) -   \nabla f_i(y))\right\|  \leq L(\textstyle\sum_{i=1}^m  p_i^*)\|x-y\| \\ & = L\|x-y\|. 
\end{align*} 
\end{remark}
\begin{remark}[Compact local representation of Alg.~\ref{alg:IncentFedAvg}]
Let us define $\mathcal{I} \triangleq \{K_1, K_2, \ldots\}$ where $K_r \triangleq T_r - 1$ for $r \geq 1$. The following equation, for $k \geq 0$, compactly represents the local update rules of Algorithm \ref{alg:IncentFedAvg}.
\begin{equation}\label{updated_x}
x_{i,k+1} := 
\begin{cases}
 \textstyle\sum_{j=1}^m \fm{p_{j,r}}\left(x_{j,k} - \gamma g_{j,k} \right), & k \in \mathcal{I} \\
x_{i,k} - \gamma g_{i,k} , & k \notin \mathcal{I}.
\end{cases}
\end{equation}
\end{remark}
{\bf Proof of Lemma~\ref{lemma:agg}}\begin{proof}  
 {Case 1.} If $k \in \mathcal{I}$, from equation \eqref{updated_x} we may write
\fm{\begin{align*}
&x_{i,k+1} = \textstyle\sum_{j=1}^m \fm{p_{j,r}}\left(x_{j,k} - \gamma g_{j,k}\right) \\
&= \textstyle\sum_{j=1}^m \fm{p_{j,r}} x_{j,k} - \gamma   \textstyle\sum_{j=1}^m \fm{p_{j,r}} g_{j,k} = \bar{x}_k - \gamma \bar{g}_k,
\end{align*}}
where the last equation is implied by the definition of $\bar{x}_k$ and $\bar{g}_k$. Taking the average on the both sides over $i \in [m]$, we obtain $\bar{x}_{k+1} = \bar{x}_k - \gamma \bar{g}_k$.

\noindent {Case 2.} If $k \notin \mathcal{I}$, from equation \eqref{updated_x}, 
$x_{i,k+1} = x_{i,k} - \gamma g_{i,k}.$ By multiplying the both sides by $\fm{p_{i,r}}$ and then, summing over $i \in [m]$, we obtain \fm{$\textstyle\sum_{i=1}^m \fm{p_{i,r}} x_{i,k+1} =  \textstyle\sum_{i=1}^m \fm{p_{i,r}} x_{i,k} - \gamma \textstyle\sum_{i=1}^m \fm{p_{i,r}} g_{i,k}.$}  Invoking Definition \ref{def:main_terms}, we obtain the result. 
\end{proof}

{\bf Proof of Lemma~\ref{linear_rate} (i)} 
\begin{proof}
\noindent {(i)} From the fixed-point property of the projected gradient method,  $N^* = \Pi_{\mathcal{N}}[ {N}^* -\tilde{\gamma}F({N}^*) ]$.  From the nonexpansivity of the Euclidean projection, we may write  
\begin{align*}
&\|\hat{N}_{r+1} - N^*\|^2  \\ &= \|\Pi_{\mathcal{N}}[ \hat{N}_{r} -\tilde{\gamma}F(\hat{N}_{r}) ]-\Pi_{\mathcal{N}}[ {N}^* -\tilde{\gamma}F({N}^*) ]\|^2\\
&\leq \|(\hat{N}_{r} -{N}^*)-\tilde\gamma(F(\hat{N}_{r})-F({N}^*))\|^2\\
& = \|\hat{N}_{r} -{N}^*\|^2 +\tilde\gamma^2\|F(\hat{N}_{r})-F({N}^*)\|^2  \\ & -2\tilde{\gamma}( \hat{N}_{r} -{N}^*)^\top(F(\hat{N}_{r})-F({N}^*))\\
&\leq (1-2\mu_F\tilde{\gamma}+L_F^2\tilde{\gamma}^2)\|\hat{N}_{r} -{N}^*\|^2,
\end{align*}
where in the preceding relation, we used the strong monotonicity and the Lipschitz continuity of $F$. From $\tilde\gamma \leq \tfrac{\mu_F}{L_F^2}$, we have $0 < 1-2\mu_F\tilde{\gamma}+L_F^2\tilde{\gamma}^2 \leq (1-0.5\mu_F\tilde\gamma)^2<1$. Thus, $\|\hat{N}_r -N^* \|\leq (1-0.5\mu_F\tilde\gamma) \|\hat{N}_r - N^*\|$. Unrolling this relation recursively, we obtain the result in (i). 
\end{proof}

{\bf Proof of Lemma~\ref{lem:weighted_gradient_inner_product}}\begin{proof}
By the linearity of expectation and rearranging the summations,
\begin{align*}
& \textstyle\sum_{i=1}^{m} p_{i,r} \mathbb{E}\left[(g_{i,t})^\top (\bar{g}_t) \mid \mathcal{F}_{T_r}\right] \\ & = \mathbb{E}\left[\left(\textstyle\sum_{i=1}^{m} p_i g_{i,t}\right)^\top (\bar{g}_t) \mid \mathcal{F}_{T_r}\right]= \mathbb{E}\left[\|\bar{g}_t\|^2 | \mathcal{F}_{T_r}\right].
\end{align*}
\end{proof}

We will make use of the following result to construct a bound on the consensus error. 
\begin{lemma}\label{lemma_recursive}  
Consider the sequence $\{a_k\}$ for $r \geq 0$, where $T_0=0$. For any given $r \geq 0$, suppose for $T_r + 1 \leq k \leq T_{r+1}$, the nonnegative sequences $\{a_k\}$ and $\{\theta_k\}$ satisfy a recursive relation of the form 
\begin{align}\label{recursive}
a_k \leq (k-T_r)\gamma^2\textstyle\sum_{t=T_r}^{k-1} (\beta a_t + \theta_t).
\end{align}
Let us assume that for ant $r \geq 0$, $T_{r+1} - T_r$ is a constant denoted by $H$, where $a_{T_r} = 0$, $\beta > 0$, and $\gamma > 0$. Then, for any $T_r + 1 \leq k \leq T_{r+1}$, we have
\begin{align}\label{first-recursive}
&a_k \leq H\gamma^2\textstyle\sum_{t=T_r}^{k-1} (\beta H\gamma^2 + 1)^{k-t-1}\theta_t
\end{align}
Moreover, if $0 < \gamma \leq \frac{1}{H\sqrt{\beta}}$, then 
$a_k \leq 3H\gamma^2\textstyle\sum_{t=T_r}^{k-1} \theta_t.$
\end{lemma}
{\bf Proof of Lemma~\ref{lemma_recursive}}  \begin{proof}
First, we prove the inequality \eqref{first-recursive} for any fixed $r \geq 0$ and all $T_r + 1 \leq k \leq T_{r+1}$. Suppose $k = T_r + 1$. From inequality \eqref{recursive}, we have
\begin{align*}
a_{T_r+1} &\leq (T_r+1-T_r)\gamma^2\textstyle\sum_{t=T_r}^{T_r} (\beta a_t + \theta_t) \\
&= \gamma^2(\beta a_{T_r} + \theta_{T_r}) = \gamma^2\theta_{T_r} \quad \text{(since $a_{T_r} = 0$)}\\
&\leq H\gamma^2\theta_{T_r} = H\gamma^2\textstyle\sum_{t=T_r}^{T_r} (\beta H\gamma^2 + 1)^{(T_r+1)-t-1}\theta_t,
\end{align*}
where we used $H \geq 1$. Thus, the inequality \eqref{first-recursive} holds for $k=T_r+1$. Assume that for all $k$, where $T_r + 1 \leq k < T_{r+1}$, the inequality \eqref{first-recursive} holds. We show that 
\begin{align*}
a_{k+1} &\leq H\gamma^2\textstyle\sum_{t=T_r}^{(k+1)-1} (\beta H\gamma^2 + 1)^{(k+1)-t-1}\theta_t \\
&=H\gamma^2\textstyle\sum_{t=T_r}^{k} (\beta H\gamma^2 + 1)^{k-t}\theta_t.
\end{align*}
From the inequality \eqref{first-recursive}, we have \fm{$a_{k+1} \leq (k+1-T_r)\gamma^2\textstyle\sum_{t=T_r}^{k} (\beta a_t + \theta_t).$} For each $a_t$ where $T_r < t \leq k$, from the inductive hypothesis, we have
$
a_t \leq H\gamma^2\sum_{s=T_r}^{t-1} (\beta H\gamma^2 + 1)^{t-s-1}\theta_s.$
Note that $a_{T_r} = 0$. Therefore, from the two preceding relations, we obtain
\begin{align} \label{Combination}
&a_{k+1} \leq  (k+1-T_r)\gamma^2\textstyle(\beta H\gamma^2 \notag\\
& \times \textstyle\sum_{t=T_r+1}^k \sum_{s=T_r}^{t-1} (\beta H\gamma^2 + 1)^{t-s-1}\theta_s + \sum_{t=T_r}^k \theta_t).
\end{align}
Rearranging the double summation in the preceding relation yields
\begin{align*}
&\textstyle\sum_{t=T_r+1}^{k} \textstyle\sum_{s=T_r}^{t-1} (\beta H\gamma^2 + 1)^{t-s-1}\theta_s \\ &= \textstyle\sum_{s=T_r}^{k-1} \textstyle\sum_{t=s+1}^{k} (\beta H\gamma^2 + 1)^{t-s-1}\theta_s \\
&= \textstyle\textstyle\sum_{s=T_r}^{k-1} \theta_s \textstyle\sum_{t=s+1}^{k} (\beta H\gamma^2 + 1)^{t-s-1}\\ &=  \textstyle\sum_{s=T_r}^{k-1} \theta_s \frac{(\beta H\gamma^2 + 1)^{k-s} - 1}{\beta H\gamma^2}.
\end{align*}
From the preceding inequality and inequality \eqref{Combination}, we obtain
\begin{align*}
a_{k+1}  &\leq (k + 1 - T_r)\gamma^2 \textstyle(\sum_{t=T_r}^{k-1} \theta_t (\beta H\gamma^2 + 1)^{k-t}  \\
& - \textstyle\sum_{t=T_r}^{k-1} \theta_t  + \textstyle\sum_{t=T_r}^{k} \theta_t\textstyle).
\end{align*}
Thus, we have \fm{$$a_{k+1}  \leq (k + 1 - T_r)\gamma^2 \textstyle\left(\sum_{t=T_r}^{k-1} \theta_t (\beta H\gamma^2 + 1)^{k-t} + \theta_k \right).$$} 
Note that when $t = k$, the expression $(\beta H\gamma^2 + 1)^{k-t} = (\beta H\gamma^2 + 1)^0 = 1$. This allows us to combine the sums \fm{$a_{k+1} \leq (k + 1 - T_r)\gamma^2 \textstyle \left(\sum_{t=T_r}^{k} \theta_t (\beta H\gamma^2 + 1)^{k-t} \right).$} 
Note that we have $(k-T_r) \leq H$. We obtain \fm{$a_{k+1} \leq H\gamma^2 \textstyle\left(\sum_{t=T_r}^{k}  (\beta H\gamma^2 + 1)^{k-t}\theta_t \right).$} 
This completes the proof of the inequality \eqref{first-recursive}. Next, we show $a_k \leq 3H\gamma^2\sum_{t=T_r}^{k-1} \theta_t$ for any fixed $r \geq 0$ and all $T_r + 1 \leq k \leq T_{r+1}$. Given the condition $0 < \gamma \leq \frac{1}{\sqrt{\beta}H}$, we have $
\beta H \gamma^2 + 1 \leq \tfrac{1}{H} + 1 = \tfrac{H+1}{H}$. For any $t$ where $T_r \leq t \leq k-1$, we know that $k-t-1 \leq T_{r+1}-T_r-1 = H-1$, since $k \leq T_{r+1}$. Thus, \fm{$(\beta H \gamma^2 + 1)^{k-t-1} \leq \left(\tfrac{H+1}{H}\right)^{H-1}$.}\\
Using the inequality $\ln(1+x) < x$ for $x > 0$, for all $H > 1$, we get $
\textstyle\ln\left[\left(\frac{H+1}{H}\right)^{H-1}\right] <   \frac{H-1}{H} < 1.$
Therefore, $\left(\frac{H+1}{H}\right)^{H-1} < \exp(1) < 3$ for all $H > 1$. This confirms that $(\beta H \gamma^2 + 1)^{k-t-1} \leq 3$. Substituting this bound back into equation \eqref{first-recursive}, we obtain the result. 
\end{proof}

{\bf Proof of Lemma~\ref{lem:Consensus_Error}}
\begin{proof} (i)
From Algorithm \ref{alg:IncentFedAvg}, we can write \fm{$ x_{i,k} = x_{i,k-1} - \gamma(g_{i,k-1})$, for all $T_r + 1 \leq k \leq T_{r+1}.$}
By applying this relation recursively, we obtain
\begin{align*}
& x_{i,k} = x_{i,T_r} - \gamma \textstyle\sum_{t=T_r}^{k-1} (g_{i,t}), \quad \text{for all } T_r + 1 \leq k \leq T_{r+1}.
\end{align*}
From Algorithm \ref{alg:IncentFedAvg}, we know $\hat{x}_r = x_{i,T_r}$. Since $\bar{x}_{T_r} = \hat{x}_r$, we have $\bar{x}_{T_r} = x_{i,T_r}$ for all clients $i$ and any round $r$. From the definition of $\bar{x}_k$ in Definition \ref{def:main_terms}, for $T_r + 1 \leq k \leq T_{r+1}$,  
\begin{align*}
&\bar{x}_k = \textstyle\sum_{i=1}^m \fm{p_{i,r}} x_{i,k} = \textstyle \sum_{i=1}^m \fm{p_{i,r}} \bar{x}_{T_r} - \gamma \textstyle\sum_{t=T_r}^{k-1} (\bar{g}_t).
\end{align*}
Invoking the definition of $\bar{e}_k$ in Definition \ref{def:main_terms}, we get
\fm{\begin{align*}
 &\mathbb{E} [\bar{e}_k | \mathcal{F}_{T_r}] = \mathbb{E} {\textstyle[\textstyle\sum_{i=1}^m \fm{p_{i,r}} \|x_{i,k} - \bar{x}_k\|^2 | \mathcal{F}_{T_r}]}  \\
 &= \textstyle\sum_{i=1}^m \fm{p_{i,r}} \mathbb{E} {\textstyle[\|x_{i,k} - \bar{x}_k\|^2 | \mathcal{F}_{T_r}]} \\
& = \textstyle\sum_{i=1}^m \fm{p_{i,r}} \mathbb{E} {\textstyle[\|\gamma \textstyle\sum_{t=T_r}^{k-1} g_{i,t} - \gamma \textstyle\sum_{t=T_r}^{k-1} \bar{g}_t\|^2 | \mathcal{F}_{T_r}]} \\
& \leq \gamma^2 (k-T_r) \textstyle\sum_{i=1}^m \fm{p_{i,r}} \textstyle\sum_{t=T_r}^{k-1} \mathbb{E} {\textstyle[\|g_{i,t} - \bar{g}_t\|^2 | \mathcal{F}_{T_r}]}.
\end{align*}}
 We may write
\begin{align*}
 & \mathbb{E} [\bar{e}_k | \mathcal{F}_{T_r}]  
\\ &\leq \gamma^2 (k-T_r) \textstyle\sum_{t=T_r}^{k-1} \textstyle\sum_{i=1}^m \fm{p_{i,r}} \mathbb{E} [\|g_{i,t} - \bar{g}_t\|^2 | \mathcal{F}_{T_r}]  \\
&= \gamma^2 (k-T_r) \textstyle\sum_{t=T_r}^{k-1} \textstyle\sum_{i=1}^m \fm{p_{i,r}} \mathbb{E} [\|g_{i,t}\|^2 | \mathcal{F}_{T_r}]  \\
&+ \gamma^2 (k-T_r) \textstyle\sum_{t=T_r}^{k-1} \textstyle\sum_{i=1}^m \fm{p_{i,r}} \mathbb{E} [\|\bar{g}_t\|^2 | \mathcal{F}_{T_r}]  \\
&- 2\gamma^2 (k-T_r) \textstyle\sum_{t=T_r}^{k-1} \textstyle\sum_{i=1}^m \fm{p_{i,r}} \mathbb{E} [g_{i,t}^{\top}\bar{g}_t | \mathcal{F}_{T_r}].
\end{align*}
\fm{Using Lemma \ref{lem:weighted_gradient_inner_product} and
noting that we have $\textstyle\sum_{i=1}^m \fm{p_{i,r}} = 1$. This implies that $\textstyle\sum_{i=1}^m \fm{p_{i,r}} \mathbb{E} [\|\bar{g}_t\|^2 | \mathcal{F}_{T_r}] = \mathbb{E} [\|\bar{g}_t\|^2 | \mathcal{F}_{T_r}]$}.  Thus, we get
\begin{align*}
\mathbb{E} [\bar{e}_k | \mathcal{F}_{T_r}]& \leq \gamma^2 (k-T_r) \textstyle\sum_{t=T_r}^{k-1} \textstyle\sum_{i=1}^m \fm{p_{i,r}} \mathbb{E} [\|g_{i,t}\|^2 | \mathcal{F}_{T_r}]  \\&- \gamma^2 (k-T_r) \textstyle\sum_{t=T_r}^{k-1} \mathbb{E} [\|\bar{g}_t\|^2 | \mathcal{F}_{T_r}].
\end{align*}
Dropping $\mathbb{E}  [\left\|\bar{g}_t\right\|^2 | \mathcal{F}_{T_r}]  $ and taking expectations on both sides we have
\begin{align}\label{eqn:e_k_intermediate}
\mathbb{E} [\bar{e}_k  ] \leq \gamma^2 (k-T_r) \textstyle\sum_{t=T_r}^{k-1} \textstyle\sum_{i=1}^m \fm{p_{i,r}} \mathbb{E}  [ \|g_{i,t} \|^2   ].
\end{align}
Next, we estimate $\mathbb{E}\left[\|g_{i,t}\|^2|\mathcal{F}_{t}\right]$. We have
\begin{align*}
g_{i,t} = \nabla \tilde{f}_i(x_{i,t}, \xi_i) &= \nabla \tilde{f}_i(x_{i,t}, \xi_i) - \nabla f_i(x_{i,t}) + \nabla f_i(x_{i,t}) \\ &+ \nabla f_i(\bar{x}_t)-\nabla f_i(\bar{x}_t).
\end{align*}
Invoking Assumption~\ref{assum:main}, we obtain $$\mathbb{E}[\|g_{i,t}\|^2|\mathcal{F}_t] \leq 3\nu^2+ 3 L^2\|  x_{i,t}- \bar{x}_t \|^2    +3 \|\nabla f_i(\bar{x}_t)\|^2 .$$
This implies that \begin{align*}
    \textstyle\sum_{i=1}^m \fm{p_{i,r}}\mathbb{E}[\|g_{i,t}\|^2|\mathcal{F}_t] &\leq 3\nu^2+ 3 L^2\bar{e}_t   \\
&+3 \textstyle\sum_{i=1}^m \fm{p_{i,r}}\|\nabla f_i(\bar{x}_t)\|^2.\end{align*}
\fy{From \eqref{eqn:e_k_intermediate} and that $p_{i,r}\leq 1$, we obtain the result in (i). 
}

\noindent \fy{(ii) Consider the preceding relation. From Assumption~\ref{assum:bgd}, we obtain}
\begin{align}\label{recursive_in_lemma5}
\mathbb{E}[\bar e_k ] &\leq \gamma^2(k - T_r) \textstyle\sum_{t=T_r}^{k-1} (3\nu^2 + 3L^2\mathbb{E}[\bar e_t] \notag\\ & +3m(G^2 + B^2\mathbb{E}[\|\nabla f(\bar{x}_t)\|^2])).
\end{align}
To complete the proof, it remains to apply Lemma~\ref{lemma_recursive} to the preceding recursive inequality. To this end, we set $a_t=\mathbb{E}[\bar e_t]$, $\beta:=3L^2$, and $ \theta_t=3 \left(\nu^2 + m(G^2 + B^2\mathbb{E}[\|\nabla f(\bar{x}_t)\|^2])\right).$ 
From Lemma~\ref{lemma_recursive}, in view of $\gamma < \frac{1}{\sqrt{3} H L}$, we obtain the result in (i).

\noindent \fy{(iii)} Summing both sides of the inequality in (i) for $k=T_{\hat{r}},\ldots, T_R-1$, and noting that $\bar{e}_{T_{\hat r}}=0$ for any $0\leq \hat{r}\leq R-1$, we obtain
\begin{align*}
 \textstyle \sum_{k=T_{\hat r}}^{T_R-1} \mathbb{E}[\bar{e}_k] &=  \textstyle \sum_{k=T_{\hat r}+1}^{T_R-1}  \mathbb{E}[\bar{e}_k]\leq  \textstyle \sum_{k=T_{\hat r}+1}^{T_R-1} \textstyle\sum_{t=T_r}^{k-1} 9H\gamma^2\\
& \times \left(\nu^2 + m(G^2 + B^2\mathbb{E}[\|\nabla f(\bar{x}_t)\|^2] )\right) \\
&\leq \textstyle 9m\gamma^2 B^2 \sum_{k=T_{\hat r}+1}^{T_R-1}   \textstyle\sum_{t=T_r}^{k-1} H \mathbb{E}[\|\nabla f(\bar{x}_t)\|^2] \\ & + 9\gamma^2 \left(\nu^2 + mG^2\right)\textstyle \sum_{k=T_{\hat r}+1}^{T_R-1}  H^2 \\
&\leq \textstyle 9m\gamma^2 B^2 H^2 \sum_{k=T_{\hat r}}^{T_R-1}  \mathbb{E}[\|\nabla f(\bar{x}_k)\|^2] \\ &  + 9\gamma^2 \left(\nu^2 + mG^2\right) H^2(T_R-T_{\hat{r}}).
\end{align*}
In the last inequality, we bound the double sum $\sum_{k=T_{\hat r}+1}^{T_R-1}  \sum_{t=T_r}^{k-1} \mathbb{E}[\|\nabla f(\bar{x}_t)\|^2]$ by observing that each  term $\mathbb{E}[\|\nabla f(\bar{x}_t)\|^2]$ appears at most $H$ times for all time indices $t$. Thus, we have $\sum_{k=T_{\hat r}+1}^{T_R-1} \sum_{t=T_r}^{k-1} \mathbb{E}[\|\nabla f(\bar{x}_t)\|^2] \leq H \sum_{k=T_{\hat r}}^{T_R-2}  \mathbb{E}[\|\nabla f(\bar{x}_k)\|^2] \leq H \sum_{k=T_{\hat r}}^{T_R-1}  \mathbb{E}[\|\nabla f(\bar{x}_k)\|^2]$.  
\end{proof}

{\bf Proof of Lemma~\ref{lemma:gbar_sq}}
\begin{proof}
(a) From Definition~\ref{def:main_terms}, for $\bar{g}_k$, we have
\begin{align}\label{eqn:barg_1}
 &\mathbb{E}[\|\bar{g}_k\|^2]  
= \textstyle \mathbb{E}[\|\sum_{i=1}^m p_{i,r} \nabla\tilde{f}_i(x_{i,k}, \xi_i)\|^2]\notag\\
&= \textstyle \mathbb{E}[\| \sum_{i=1}^m p_{i,r} \left(\nabla\tilde{f}_i(x_{i,k}, \xi_i) - \nabla f_i(x_{i,k})\right)\|^2] \notag\\ &  + \textstyle\mathbb{E}[\| \sum_{i=1}^m p_{i,r} \nabla f_i(x_{i,k})\|^2],
\end{align}
where the preceding equation is implied in view of 
\begin{align*}
 & \mathbb{E}[(\nabla\tilde{f}_i(x_{i,k}, \xi_i) - \nabla f_i(x_{i,k}))^\top \nabla f_i(x_{i,k})]  
\\ & = \mathbb{E}[\mathbb{E}[(\nabla\tilde{f}_i(x_{i,k}, \xi_i) - \nabla f_i(x_{i,k}))^\top \nabla f_i(x_{i,k})\mid \mathcal{F}_k]]\\
& = \mathbb{E}[\mathbb{E}[(\nabla\tilde{f}_i(x_{i,k}, \xi_i) - \nabla f_i(x_{i,k}))\mid \mathcal{F}_k]^\top \nabla f_i(x_{i,k})] =0.
\end{align*}
Invoking Assumption~\ref{assum:main}, from \eqref{eqn:barg_1}, we obtain
\begin{align}\label{eqn:barg_2}
 \mathbb{E}[\|\bar{g}_k\|^2]  
&= \textstyle \sum_{i=1}^m p_{i,r}^2\mathbb{E}[\| (\nabla\tilde{f}_i(x_{i,k}, \xi_i) - \nabla f_i(x_{i,k}) )\|^2] \notag\\
& + \textstyle\mathbb{E}[\| \sum_{i=1}^m p_{i,r} \nabla f_i(x_{i,k})\|^2] \notag\\
& \leq \textstyle \sum_{i=1}^m p_{i,r}^2 \nu^2 + \textstyle\mathbb{E}[\| \sum_{i=1}^m p_{i,r} \nabla f_i(x_{i,k})\|^2] \notag\\ 
& \leq  \fm{\nu^2 + \textstyle\mathbb{E}[\| \sum_{i=1}^m p_{i,r} \nabla f_i(x_{i,k})\|^2]}.
\end{align}Next, we construct an upper bound on the term 
$\textstyle\mathbb{E}[\| \sum_{i=1}^m p_{i,r} \nabla f_i(x_{i,k})\|^2]$. 
We write
\begin{align*}
&\textstyle\mathbb{E}[\| \sum_{i=1}^m p_{i,r} \nabla f_i(x_{i,k})\|^2]  \\
&= \textstyle \mathbb{E}[\|\sum_{i=1}^{m}p_{i,r}(
\nabla f_i(x_{i,k}) - \nabla f_i(\bar{x}_k)
+ \nabla f_i(\bar{x}_k))\|^2].
\end{align*}
Applying the identity $\|\textstyle\sum_{t=1}^{T}y_t \|^2 \leq T \textstyle\sum_{t=1}^{T}\|y_t\|^2$, 
\begin{align}\label{eqn:barg_3}
& \textstyle\mathbb{E}[\| \sum_{i=1}^m p_{i,r} 
\nabla f_i(x_{i,k})\|^2]
\notag   \leq 2\,
\mathbb{E}[\|
\textstyle\sum_{i=1}^{m}p_{i,r}\nabla f_i(\bar{x}_k)
\|^2] \notag \\ 
& +2m\textstyle\sum_{i=1}^{m}p_{i,r}^2
\mathbb{E}[\|
\nabla f_i(x_{i,k})-\nabla f_i(\bar{x}_k)
\|^2].
\end{align}
For the second term we have
\begin{align}\label{eqn:barg_4}
&\textstyle\mathbb{E}[\|
 \textstyle\sum_{i=1}^{m}p_{i,r}\nabla f_i(\bar{x}_k)
\|^2]
\notag \\
& =
\textstyle\mathbb{E}[\|
\textstyle\sum_{i=1}^{m}(p_{i,r}-p_i^*)
\nabla f_i(\bar{x}_k)    +
\textstyle\sum_{i=1}^{m}p_i^*\nabla f_i(\bar{x}_k)
\|^2] \notag\\
&\leq
2\,\textstyle\mathbb{E}[\|
\textstyle\sum_{i=1}^{m}(p_{i,r}-p_i^*)
\nabla f_i(\bar{x}_k)
\|^2] \notag\\
&\quad+
2\,\textstyle\mathbb{E}[\|
\textstyle\sum_{i=1}^{m}p_i^*\nabla f_i(\bar{x}_k)
\|^2] \notag\\
&\leq
2m\,\textstyle\sum_{i=1}^{m}(p_{i,r}-p_i^*)^2
\mathbb{E}[\|\nabla f_i(\bar{x}_k)\|^2]  +
2\,\textstyle\mathbb{E}[\|\nabla f(\bar{x}_k)\|^2] \notag\\
&\leq
2m\fy{\delta_r^2}\,
\textstyle\sum_{i=1}^{m}
\mathbb{E}[\|\nabla f_i(\bar{x}_k)\|^2]
+
2\,\textstyle\mathbb{E}[\|\nabla f(\bar{x}_k)\|^2].
\end{align}
Utilizing the Lipschitz continuity of the local gradients for the first term in 
\eqref{eqn:barg_3} and that $p_{i,r} \leq 1$, from \eqref{eqn:barg_3} and 
\eqref{eqn:barg_4}, we may write
\begin{align}\label{eqn:barg_5}
& \textstyle\mathbb{E}[\|
\textstyle\sum_{i=1}^m p_{i,r} \nabla f_i(x_{i,k})
\|^2]
\leq
2mL^2 \mathbb{E}[\bar{e}_k]
\notag\\
& +4m\fy{\delta_r^2}
\textstyle\sum_{i=1}^{m}
\mathbb{E}[\|\nabla f_i(\bar{x}_k)\|^2]
+
4\,\textstyle\mathbb{E}[\|\nabla f(\bar{x}_k)\|^2].
\end{align}
\fm{The bound in part (a) follows by combining \eqref{eqn:barg_5} with \eqref{eqn:barg_2}.}

\noindent (b) From Assumption~\ref{assum:bgd} and  \eqref{eqn:barg_5}, we obtain
\begin{align*}
\textstyle\mathbb{E}[\|
\textstyle\sum_{i=1}^m p_{i,r} \nabla f_i(x_{i,k})
\|^2]
&\leq
2mL^2 \mathbb{E}[\bar{e}_k]
+4m^2\fy{\delta_r^2}G^2 \\
& +
4(m^2\fy{\delta_r^2}B^2+1)
\mathbb{E}[\|\nabla f(\bar{x}_k)\|^2].
\end{align*}
\fm{Combining the preceding bound with \eqref{eqn:barg_2} completes the proof of part (b).}
\end{proof}

\subsection{\fm{Convex Settings}}

\fm{In this section, we extend our analysis to the case where the local loss functions are convex. This setting allows for convergence to the optimal solution rather than just stationary points.}\\
\fy{In the analysis of the convex setting, we will utilize the following definition.}
\begin{definition}\label{D_f}
For a function $f$
and any arbitrary points $x$ and $y$, the associated Bregman divergence is defined by $
D_f(x, y) \triangleq f(x) - f(y) - \nabla f(y)^\top(x - y).$
\end{definition}

\begin{lemma}[BGD in convex setting]\label{relaxedBGD}
\fy{Let Assumptions~\ref{assum:game} and~\ref{assum:c} hold. Then, for all $x \in \mathbb{R}^n$, we have $
\tfrac{1}{m}\textstyle\sum_{i=1}^m \|\nabla f_i(x)\|^2 \leq G^2 + 2LB^2\left(f(x) - f^*\right),$ where $G^2 := 2L(f^* - \tfrac{1}{m}\textstyle\sum_{i=1}^m f_i^*) 
+ \textstyle 2L\left(\max_{i \in [m]}\left|\tfrac{1}{mp_i^*}-1\right|\right)f^*$ and 
$B^2 := 1 + \max_{i \in [m]}\left|\tfrac{1}{mp_i^*}-1\right|$.}
\end{lemma}

\begin{proof}
Let $x_i^* \in \arg\min_{x \in \mathbb{R}^n} f_i(x)$ for each $i \in [m]$. We may write
\begin{align*}
&\|\nabla f_i(x)\|^2= \|\nabla f_i(x) - \nabla f_i(x_i^*)\|^2 \leq 2L\, D_{f_i}(x, x_i^*) \\
&= 2L(f_i(x)-f_i^*).
\end{align*}
Averaging the both sides over $i \in [m]$, we get $
\tfrac{1}{m}\textstyle\sum_{i=1}^m \|\nabla f_i(x)\|^2 
\leq \frac{2L}{m} \sum_{i=1}^m \left(f_i(x) - f_i^*\right).$
\fm{Adding and subtracting $2L\sum_{i=1}^m p_i^* f_i(x)$ on the right-hand side and using $f(x)\triangleq\sum_{i=1}^m p_i^* f_i(x)$, we may write}
\begin{align*}
&\tfrac{1}{m}\textstyle\sum_{i=1}^m \|\nabla f_i(x)\|^2 
\leq 2L\textstyle\sum_{i=1}^m \left(\frac{1}{m} - p_i^*\right) f_i(x) 
\\
&+ \textstyle2L\left(f(x) - \frac{1}{m}\sum_{i=1}^m f_i^*\right)\leq \textstyle 2L\textstyle\sum_{i=1}^m \left|\frac{1}{m} - p_i^*\right| f_i(x) \\
&+2L\left(f(x) -f^*\right) +\textstyle2L\left(f^*- \frac{1}{m}\sum_{i=1}^m f_i^*\right)\\
& \leq \textstyle2L \left(\max_{i \in [m]} \left|\tfrac{1}{mp_i^*} - 1\right|\right) f(x) +2L\left(f(x) -f^*\right) \\
&+2L\left(f^*- \tfrac{1}{m}\textstyle\sum_{i=1}^m f_i^*\right).
\end{align*}
Rearranging the terms, we obtain the result. 
\end{proof}

\begin{lemma}\label{lem:variance_bound_adapted}
Consider Algorithm~\ref{alg:IncentFedAvg}. Let Assumptions~\ref{assum:game} and \ref{assum:c}. For any $k$ where $T_r \leq k \leq T_{r+1}-1$, we have
\fm{\begin{align*}
\mathbb{E}[\|\bar{g}_k\|^2] &\leq \nu^2 + 2mL^2 \mathbb{E}\left[ \bar{e}_k \right] + 4m^2\delta_r^2G^2  \\
&+ 8L\left(m^2\delta_r^2B^2+1\right) \mathbb{E}\left[D_f(\bar{x}_k,x^*)\right].
\end{align*}}
\end{lemma}
\begin{proof}
Consider Lemma \ref{lemma:gbar_sq} (a).
{Invoking Lemma \ref{relaxedBGD}, we get}
\begin{align*}
&\mathbb{E}[\|\bar{g}_k\|^2]  \leq \nu^2 + 2mL^2 \mathbb{E}\left[ \bar{e}_k \right] + 4m^2\delta_r^2G^2  \\
&+ 8m^2\delta_r^2L B^2 \mathbb{E}\left[f(\bar{x}_k) - f^*\right] +4 \mathbb{E}\left[\| \nabla f(\bar{x}_k)\|^2\right].
\end{align*}
Since $\nabla f(x^*)=0$ , using Definition \ref{D_f} we have  $\mathbb{E}\left[f(\bar{x}_k) - f^*\right]=\mathbb{E}\left[D_f(\bar{x}_k,x^*)\right]$, Utilizing $\mathbb{E}\left[\|\nabla f(\bar{x}_k) \|^2 \right] \leq 2 L \mathbb{E}\left[D_f(\bar{x}_k,x^*)\right]$, we obtain the result. 
\end{proof}
\begin{lemma}\label{lem:weighted-hetero-convex} \em
\fm{Consider Algorithm~\ref{alg:IncentFedAvg}. Let Assumptions~\ref{assum:game} and Assumption~\ref{assum:c} hold. For any $k$ where $T_r \leq k \leq T_{r+1}-1$, we have} 
 \fm{ \begin{align*}
&-2\textstyle\sum_{i=1}^{m} p_{i,r} (\bar{x}_k - x^*)^\top \nabla f_i(x_{i,k}) \\
&\leq   -2(1-L B^2\delta_r\sqrt{\delta_r}) D_f(\bar{x}_k,x^*)\\
&+ m\sqrt{\delta_r} \textstyle \|\bar{x}_k - x^*\|^2   \\
&+ \delta_r\sqrt{\delta_r} \left(G^2+ \textstyle\sum_{i=1}^m \|\nabla f_i(x^*)\|^2\right)    +L\bar{e}_k.
\end{align*}}
\end{lemma}
\begin{proof}
\fm{We may write
\begin{align*}
&-2\textstyle\sum_{i=1}^m p_{i,r}\left( \bar{x}_k-x^*\right)^{\top} \nabla f_i(x_{i,k})\\
&= -2\textstyle\sum_{i=1}^m p_{i,r} \left( x_{i,k}-x^*\right)^{\top}\nabla f_i(x_{i,k}) \\
&\phantom{=} -2\textstyle\sum_{i=1}^m p_{i,r}\left( \bar{x}_k-x_{i,k}\right)^{\top} \nabla f_i(x_{i,k}).
\end{align*}}
\fm{Consider the preceding relation. For the first term on the right, by convexity of $f_i$, for any $i \in [m]$, we have 
$-\left( x_{i,k}-x^*\right)^{\top}\nabla f_i(x_{i,k})\leq f_i(x^*)-f_i(x_{i,k}).$
For the second term on the right, by $L$-smoothness, we have}
\fm{\begin{align*}
&-\left( \bar{x}_k-x_{i,k}\right)^{\top} \nabla f_i(x_{i,k}) \\
&\leq f_i(x_{i,k})-f_i(\bar{x}_k)+\tfrac{L}{2}\|\bar{x}_k-x_{i,k}\|^2.
\end{align*}}
\fm{From the preceding relations, we obtain}
\fm{\begin{align*}
&-2\textstyle\sum_{i=1}^{m} p_{i,r} (\bar{x}_k - x^*)^\top \nabla f_i(x_{i,k}) \\
&\leq 2\textstyle\sum_{i=1}^{m} p_{i,r}(f_i(x^*) - f_i(\bar{x}_k)) 
+ L\textstyle\sum_{i=1}^{m} p_{i,r} \|x_{i,k} - \bar{x}_k\|^2\\
& = 2\textstyle\sum_{i=1}^{m} p^*_i(f_i(x^*) - f_i(\bar{x}_k)) \\
&+ \textstyle2\sum_{i=1}^{m} (p_{i,r} - p^*_i)(f_i(x^*) - f_i(\bar{x}_k)) +L\bar{e}_k\\
& \leq  2(f(x^*) - f(\bar{x}_k)) + 2\textstyle\sum_{i=1}^{m} |p_{i,r} - p^*_i||f_i(x^*) - f_i(\bar{x}_k)| \\
&+L\bar{e}_k.
\end{align*}}
\fm{Invoking Lemma~\ref{linear_rate}, Definition~\ref{def:main_terms}, and the definition of Bregman divergence, we obtain }
\fm{\begin{align}\label{ineq:p_ir_barxnabla_convex}
&-2\textstyle\sum_{i=1}^{m} p_{i,r} (\bar{x}_k - x^*)^\top \nabla f_i(x_{i,k})\nonumber\\ 
&\leq -2 D_f(\bar{x}_k,x^*) + 2\delta_r\textstyle\sum_{i=1}^{m} |f_i(x^*) - f_i(\bar{x}_k)| +L\bar{e}_k.
\end{align}}
\fm{From Young's inequality we have}
\fm{\begin{align*}
    & f_i(x^*) - f_i(\bar{x}_k) \leq \nabla f_i(\bar{x}_k)^\top (x^* - \bar{x}_k) 
    \\
    &\leq \tfrac{\sqrt{\delta_r}}{2}\|\nabla f_i(\bar{x}_k)\|^2 +\tfrac{1}{2\sqrt{\delta_r}}\|x^* - \bar{x}_k\|^2  \\
   \text{and  } & f_i(\bar{x}_k)-f_i(x^*)  \leq \nabla f_i(x^*)^\top (  \bar{x}_k-x^*)
   \\
   &\leq \tfrac{\sqrt{\delta_r}}{2}\|\nabla f_i(x^*)\|^2 +\tfrac{1}{2\sqrt{\delta_r}}\|x^* - \bar{x}_k\|^2.
\end{align*}}
\fm{Thus, we have}
\fm{\begin{align*}
     |f_i(x^*) - f_i(\bar{x}_k)|  
    &\leq \tfrac{\sqrt{\delta_r}}{2} \left(\|\nabla f_i(x^*)\|^2 +\|\nabla f_i(\bar{x}_k)\|^2\right) 
    \\
    &+\tfrac{1}{2\sqrt{\delta_r}}\|\bar{x}_k-x^*\|^2.
\end{align*}}
\fm{Invoking Lemma~\ref{relaxedBGD} and the Bregman divergence, we obtain} 
\fm{\begin{align*}
&2\delta_r\textstyle\sum_{i=1}^m|f_i(x^*) - f_i(\bar{x}_k)| \\
&\leq  \delta_r\sqrt{\delta_r}  \textstyle\sum_{i=1}^m\left(\|\nabla f_i(x^*)\|^2 +\|\nabla f_i(\bar{x}_k)\|^2\right) 
\\
&+ m\sqrt{\delta_r}\|\bar{x}_k-x^*\|^2\\
&\leq  \delta_r\sqrt{\delta_r}  \left(\textstyle\sum_{i=1}^m \|\nabla f_i(x^*)\|^2  + G^2 + 2 L B^2 D_f(\bar{x}_k,x^*)\right) 
\\
&+ m\sqrt{\delta_r}\|\bar{x}_k-x^*\|^2.
\end{align*}}
\fm{From the preceding relation and \eqref{ineq:p_ir_barxnabla_convex}, we obtain the result.}
\end{proof}
\begin{lemma}\label{lem:optgap_convex} \em
\fm{Consider Algorithm~\ref{alg:IncentFedAvg}. Let Assumptions~\ref{assum:game} and \ref{assum:c} hold. Suppose $\gamma \leq \min\{\frac{1}{16L m^2B^2\delta_{0}},\frac{1}{16L}\}$ and $r \geq  \frac{\delta_{0}\sqrt[3]{16L^2B^4}}{\ln(1/\rho)}$ where $\delta_{0}$ is given by Lemma~\ref{linear_rate}. Then, for any $k$ where $T_r \leq k \leq T_{r+1}-1$, we have}
\fm{\begin{align*}
&\mathbb{E}[\|\bar{x}_{k+1}-x^*\|^2|\mathcal{F}_k] \leq  \left(1+m\gamma\sqrt{\delta_r}\right)\|\bar{x}_k-x^*\|^2 
 \notag   \\
 &-\tfrac{\gamma}{2}  D_f(\bar{x}_k,x^*)    \notag   + \gamma\delta_r\sqrt{\delta_r} \left(G^2+ \textstyle\sum_{i=1}^m \|\nabla f_i(x^*)\|^2\right) \\
 &+ 4\gamma^2m^2 \delta_r^2 G^2 + L\gamma(2m L\gamma   +1)\bar{e}_k + \gamma^2\nu^2  .
\end{align*}}
\end{lemma}
\begin{proof}
\fm{Recall from Lemma~\ref{lemma:agg} that $\bar{x}_{k+1} = \bar{x}_k - \gamma\bar{g}_k$. We have}
\fm{\begin{align*}
&\|\bar{x}_{k+1}-x^*\|^2 = \|\bar{x}_k - x^* - \gamma\bar{g}_k\|^2 = \|\bar{x}_k-x^*\|^2 \\
&+ \gamma^2\|\bar{g}_k\|^2 - 2\gamma \left( \bar{x}_k-x^*\right)^{\top}  \bar{g}_k.
\end{align*}}
\fm{Taking expectations conditioned on $\mathcal{F}_k$, we obtain}
\fm{\begin{align*}
&\mathbb{E}[\|\bar{x}_{k+1}-x^*\|^2|\mathcal{F}_k] \leq  \|\bar{x}_k-x^*\|^2  + \gamma^2\mathbb{E}[\|\bar{g}_k\|^2|\mathcal{F}_k] \\
&- 2\gamma\,\mathbb{E}[\left( \bar{x}_k-x^*\right)^{\top}  \bar{g}_k|\mathcal{F}_k].
\end{align*}}
\fm{Invoking Lemmas~\ref{lem:variance_bound_adapted} and~\ref{lem:weighted-hetero-convex}, we obtain}
\fm{\begin{align*}
&\mathbb{E}[\|\bar{x}_{k+1}-x^*\|^2|\mathcal{F}_k] \leq  \left(1+m\gamma\sqrt{\delta_r}\right)\|\bar{x}_k-x^*\|^2 
 \notag   \\
&-2\gamma(1-L B^2\delta_r\sqrt{\delta_r} -4L\gamma( m^2 \delta_r^2  B^2 + 1 )) D_f(\bar{x}_k,x^*)    \notag   \\
&+ \gamma\delta_r\sqrt{\delta_r} \left(G^2+ \textstyle\sum_{i=1}^m \|\nabla f_i(x^*)\|^2\right) + 4\gamma^2m^2 \delta_r^2 G^2     \\
&+ L\gamma(2m L\gamma   +1)\bar{e}_k + \gamma^2\nu^2.
\end{align*}}
\fm{Recall the definition of $\delta_r$ in Lemma~\ref{linear_rate}. Note that $\delta_r = \delta_{0}\rho^r$ and so, $r \geq  \frac{\delta_{0}\sqrt[3]{16L^2B^4}}{\ln(1/\rho)}$ implies that $\delta_r \leq \frac{1}{\sqrt[3]{16L^2B^4}}$.  From $\delta_r \leq \frac{1}{\sqrt[3]{16L^2B^4}}$, $\gamma \leq \min\{\frac{1}{16L m^2B^2\delta_{0}},\frac{1}{16L}\}$, and $\delta_r \leq \delta_{0}$, we have $1-L B^2\delta_r\sqrt{\delta_r} -4L\gamma( m^2 \delta_r^2  B^2 + 1 ) \geq \frac{1}{4}$. This completes the proof.}
\end{proof}
\begin{lemma}\label{lem:variance_accum_convex} \em
\fy{Consider Algorithm~\ref{alg:IncentFedAvg}. Let Assumptions~\ref{assum:game} and \ref{assum:c} hold and $\gamma \leq \frac{1}{\sqrt{3} H L}$. Then, for any communication round $ 0 \leq r \leq R-1$, we have}
\begin{align*}
&\textstyle\sum_{k=T_r}^{T_{r+1}-1}\mathbb{E}[\bar{e}_k\fm{|\mathcal{F}_k}] \\
&\leq  18mL\gamma^2 B^2 H^2 \textstyle\sum_{t=T_r+1}^{T_{r+1}-1}  \mathbb{E}[D_f(\bar{x}_t, x^*)] \\
&+ 9\gamma^2 \left(\nu^2 + mG^2\right) H^3.
\end{align*}
\end{lemma}
\begin{proof}
\fy{Consider Lemma\ref{lem:Consensus_Error} (i). Invoking \fm{Lemma}~\ref{relaxedBGD} and using $\mathbb{E}\left[f(\bar{x}_k) - f^*\right]=\mathbb{E}\left[D_f(\bar{x}_k,x^*)\right]$, we may write}
{\begin{align}\label{recursive_in_lemma5}
\mathbb{E}[\bar e_k ] &\leq \gamma^2(H-1) \textstyle\sum_{t=T_r}^{k-1} (3\nu^2 + 3L^2\mathbb{E}[\bar e_t] \nonumber\\
&+3m(G^2 + 2L B^2\mathbb{E}[D_f(\bar{x}_t, x^*)]).
\end{align}}
\fm{To complete the proof, \fy{we} apply Lemma~\ref{lemma_recursive} to the preceding \fy{relation by setting} $a_t=\mathbb{E}[\bar e_t]$, $\beta:=3L^2$, and $ \theta_t=3 \left(\nu^2 + m(G^2 + 2L B^2\mathbb{E}[D_f(\bar{x}_t, x^*)])\right).$ From Lemma~\ref{lemma_recursive} and in view of $\gamma \leq \frac{1}{\sqrt{3} H L}$, we get} 
\begin{align*}
\mathbb{E}[\bar{e}_k ] \leq 9H\gamma^2 \textstyle\sum_{t=T_r}^{k-1}\left(\nu^2 + m(G^2 + 2L B^2\mathbb{E}[D_f(\bar{x}_t, x^*)])\right)
\end{align*}
{Summing both sides for $k=T_r,\ldots, T_{r+1}-1$, and noting that $\bar{e}_{T_r}=0$, we obtain}
{\begin{align*}
&\textstyle \sum_{k=T_r}^{T_{r+1}-1} \mathbb{E}[\bar{e}_k] =  \textstyle \sum_{k=T_r+1}^{T_{r+1}-1}  \mathbb{E}[\bar{e}_k ]\\
&\leq  \textstyle \sum_{k=T_r+1}^{T_{r+1}-1} \textstyle\sum_{t=T_r}^{k-1} 9H\gamma^2\left(\nu^2 + m(G^2 \right.\\
&\left.+ 2L B^2\mathbb{E}[D_f(\bar{x}_t, x^*)])\right) \\
&\leq \textstyle 18mL\gamma^2 B^2 \sum_{k=T_r+1}^{T_{r+1}-1}   \textstyle\sum_{t=T_r}^{k-1} H \mathbb{E}[D_f(\bar{x}_t, x^*)]) \\
&+ 9\gamma^2 \left(\nu^2 + mG^2\right)\textstyle \sum_{k=T_r+1}^{T_{r+1}-1}  H^2. 
\end{align*}}
{Next, we bound the double sum $\sum_{k=T_r+1}^{T_{r+1}-1}  \sum_{t=T_r}^{k-1} \mathbb{E}[D_f(\bar{x}_t, x^*)]$ by noting that each  term $\mathbb{E}[D_f(\bar{x}_t, x^*)]$ appears at most $H$ times for all $t$. We obtain
\begin{align*}
&\textstyle\sum_{k=T_r+1}^{T_{r+1}-1} \sum_{t=T_r}^{k-1} \mathbb{E}[D_f(\bar{x}_t, x^*)] \\
&\leq H \textstyle\sum_{t=T_r}^{T_{r+1}-2}  \mathbb{E}[D_f(\bar{x}_t, x^*)] \leq H \textstyle\sum_{t=T_r+1}^{T_{r+1}-1}  \mathbb{E}[D_f(\bar{x}_t, x^*)].
\end{align*}
This completes the proof.}
\end{proof}
\begin{lemma}\label{lem:cmin_bound} Consider Algorithm~\ref{alg:IncentFedAvg}. Let Assumptions~\ref{assum:game} and Assumption~\ref{assum:c} hold. For $k\geq 1$, let  $
c_k \triangleq \frac{1}{\prod_{j=0}^{k-1}(1+\fy{m\gamma\sqrt{\delta_{r_j}}})},$ where $r_j = \lfloor j/H \rfloor$ denotes the communication round at iteration $j$. Then, for all $ 0 \leq k \leq T_R$, $
c_k \geq c_{\min} \triangleq   \fy{1-\left(\frac{m H\sqrt{\delta_{0}}}{1-\rho^{0.5}}\right)\gamma}$ .
\end{lemma}
\begin{proof}
\fy{Recall that $\delta_{r_j} = \fy{\delta_0} \rho^{r_j}$.} We have $
\ln(c_k) = -\sum_{j=0}^{k-1}\ln(1+\fy{m\gamma\sqrt{\delta_{r_j}}}).$
Using the inequality $\ln(1+x) \leq x$ for all $x > -1$, we obtain $\ln(c_k) \geq -\textstyle\sum_{j=0}^{k-1}\fy{m\gamma\sqrt{\delta_{0}\rho^{r_j} }} \geq -m\gamma\sqrt{\delta_{0}}\sum_{j=0}^{\infty}\sqrt{\rho}^{r_j}.$
Since there are $H$ iterations per communication round, we have
$\textstyle\sum_{j=0}^{\infty}\rho^{0.5r_j}  =H\textstyle\sum_{r=0}^{\infty}\rho^{0.5r}  < \tfrac{H}{1-\rho^{0.5}}.$
From the preceding two relations, we obtain $c_k\geq \exp\left(-\frac{mH\gamma\sqrt{\delta_{0}}}{1-\rho^{0.5}}\right)$. \fy{The result follows by invoking the identity $\exp(-x) \geq 1-x$ for any $x>0$. }
\end{proof}
{\bf Proof of Theorem~\ref{thm:convex}.}
\noindent (i) Recall the sequence $\{c_k\}$ and $c_{\min}$ given in Lemma~\ref{lem:cmin_bound}. Multiplying both sides of the inequality in Lemma~\ref{lem:optgap_convex} by $c_{k}$ and noting that $c_{\min}\leq c_{k+1} \leq c_k <1$, we may write
\fm{\begin{align*}
&\tfrac{c_{\min}\gamma}{2}\mathbb{E}[D_f(\bar{x}_k, x^*)] \leq c_k\mathbb{E}[\|\bar{x}_k-x^*\|^2] \\
&- c_{k+1}\mathbb{E}[\|\bar{x}_{k+1}-x^*\|^2] + L\gamma(2m L\gamma + 1)\mathbb{E}[\bar{e}_k] + \fm{\gamma^2\nu^2} \\
&+ \fm{4}\gamma^2 m^2\fm{\delta_r^2}G^2 + \gamma\fm{\delta_r\sqrt{\delta_r}(G^2 + \textstyle\sum_{i=1}^m \|\nabla f_i(x^*)\|^2)}.
\end{align*}}
Summing from $k = T_{\hat{r}}$ to $T_R-1$, we have
\begin{align}\label{eq:after_telescope}
&\tfrac{c_{\min}\gamma}{2}\textstyle\sum_{k=T_{\hat{r}}}^{T_R-1}\mathbb{E}[D_f(\bar{x}_k, x^*)] \leq \mathbb{E}[\|\bar{x}_{T_{\hat{r}}}-x^*\|^2] \nonumber\\
&+ L\gamma(2m L\gamma + 1)\textstyle\sum_{k=T_{\hat{r}}}^{T_R-1}\mathbb{E}[\bar{e}_k] \notag\\
& + \fm{\gamma^2\nu^2}(T_R-1 - T_{\hat{r}} + 1) + \fm{4}\gamma^2 m^2\fm{\fy{\delta_{0}}^2}G^2\textstyle\sum_{k=T_{\hat{r}}}^{T_R-1}\rho^{2r} \notag\\
& + \gamma\fm{\fy{\delta_{0}}\sqrt{\fy{\delta_{0}}}(G^2 + \textstyle\sum_{i=1}^m \|\nabla f_i(x^*)\|^2)}\textstyle\sum_{k=T_{\hat{r}}}^{T_R-1}\rho^{1.5r},
\end{align}
Where we dropped the nonpositive term $-c_{T_R}\mathbb{E}[\|\bar{x}_{T_R}-x^*\|^2]$ and utilized $c_{T_{\hat{r}}} \leq 1$. Invoking the bound on $\sum_{k=T_{\hat{r}}}^{T_R-1}\mathbb{E}[\bar{e}_k]$ given in Lemma~\ref{lem:variance_accum_convex}, we obtain
\begin{align*}
&\tfrac{c_{\min}\gamma}{2}\textstyle\sum_{k=T_{\hat{r}}}^{T_R-1}\mathbb{E}[D_f(\bar{x}_k, x^*)] \leq \mathbb{E}[\|\bar{x}_{T_{\hat{r}}}-x^*\|^2] \\
&+ 18mL^2\gamma^3(2m L\gamma + 1) B^2 H^2 \textstyle\sum_{k=T_{\hat{r}}}^{T_R-1}\mathbb{E}[D_f(\bar{x}_k, x^*)] \\
& + 9\gamma^3 L(2m L\gamma + 1)(\nu^2 + mG^2)H^3 + \fm{\gamma^2\nu^2}(T_R - T_{\hat{r}} ) \\
& + \fm{4}\gamma^2 m^2\fm{\fy{\delta_{0}}^2}G^2\textstyle\sum_{k=T_{\hat{r}}}^{T_R-1}\rho^{2r} \\
&+ \gamma\fm{\fy{\delta_{0}}\sqrt{\fy{\delta_{0}}}(G^2 + \textstyle\sum_{i=1}^m \|\nabla f_i(x^*)\|^2)}\textstyle\sum_{k=T_{\hat{r}}}^{T_R-1}\rho^{1.5r},
\end{align*}
Where we used $\sum_{t=T_{\hat{r}}+1}^{T_R-1}\mathbb{E}[D_f(\bar{x}_t, x^*)] \leq \sum_{k=T_{\hat{r}}}^{T_R-1}\mathbb{E}[D_f(\bar{x}_k, x^*)]$. \fy{Next, we show that $18mL^2\gamma^3(2m L\gamma + 1) B^2 H^2\leq \fy{\tfrac{c_{\min}\gamma}{4}}$. From $\gamma \leq \frac{1}{2mL}$, we have  $18mL^2\gamma^3(2m L\gamma + 1) B^2 H^2 \leq 36mL^2\gamma^3 B^2 H^2$. Consider the identity that $\gamma \leq \tfrac{1}{a+b}$ guarantees $b^2\gamma^2 \leq 4(1-a\gamma)$ where $a,b,\gamma>0$. Thus, from the assumption that $\gamma \leq \frac{1}{mH\sqrt{\delta_{0}}(1-\rho^{0.5})^{-1}+24LHB\sqrt{m}}$, we have $18mL^2\gamma^3(2m L\gamma + 1) B^2 H^2\leq \fy{\tfrac{c_{\min}\gamma}{4}}$. Thus, we obtain}
\begin{align*}
&\tfrac{c_{\min}\gamma}{4}\textstyle\sum_{k=T_{\hat{r}}}^{T_R-1}\mathbb{E}[D_f(\bar{x}_k, x^*)] \leq \fm{\mathbb{E}[\|\bar{x}_{T_{\hat{r}}}-x^*\|^2]}\\
&+ 18\gamma^3 L(\nu^2 + mG^2)H^3 + \fm{\gamma^2\nu^2} (T_R-T_{\hat{r}}) \\
& + \fm{4}\gamma^2 m^2\fm{\fy{\delta_{0}}^2}G^2 H\tfrac{1-\rho^{2R}}{1-\rho^2}\\
&+ \gamma\fm{\fy{\delta_{0}}\sqrt{\fy{\delta_{0}}}(G^2 + \textstyle\sum_{i=1}^m \|\nabla f_i(x^*)\|^2)} H\tfrac{1-\rho^{1.5R}}{1-\rho^{1.5}}.
\end{align*}
Multiplying both sides by $\frac{4}{\gamma c_{\min}}$ and noting that, in view of convexity of $f$ and Jensen's inequality,  $
\mathbb{E}[D_f(\bar{x}_T^{\text{avg}}, x^*)] \leq \tfrac{1}{T_R-T_{\hat{r}}}\textstyle\sum_{k=T_{\hat{r}}}^{T_R-1}\mathbb{E}[D_f(\bar{x}_k, x^*)],$ we obtain the result. \\ 
\noindent (ii)  \fy{From $\gamma \leq \frac{1}{2mL}$, we have $\left(1-\tfrac{\fm{m}\gamma \fm{\sqrt{\fy{\delta_{0}}}} H}{1-\rho^{0.5}}\right)^{-1} \leq \left(1-\tfrac{  \fm{\sqrt{\fy{\delta_{0}}}} H}{2L(1-\rho^{0.5})}\right)^{-1}\triangleq \hat{b}$. Consider the inequality in part~(i). Noting that  $\frac{1}{\sqrt{R_\varepsilon H} - T_{\hat{r}}/\sqrt{R_\varepsilon H}} \leq \frac{2}{\sqrt{R_\varepsilon H}}$ for $R_\varepsilon \geq 4T_{\hat{r}}^2/H$, We have 
\begin{align}\label{eqn:thm2_eq1}
\tfrac{4\fm{D_{T_{\hat{r}}}^2}}{\gamma   (T_R-T_{\hat{r}})} = \tfrac{4\fm{D_{T_{\hat{r}}}^2}}{ ((R_\varepsilon H)^{1/2} - T_{\hat{r}} (R_\varepsilon H)^{-1/2})} \leq \tfrac{8\fm{D_{T_{\hat{r}}}^2}}{ (R_\varepsilon H)^{1/2}}.
\end{align}  For $R_\varepsilon \geq 2T_{\hat{r}}/H$, we may write \begin{align}\label{eqn:thm2_eq2}\tfrac{72\gamma^2 L(\nu^2 + mG^2)H^3}{ (T_R-T_{\hat{r}})} = \tfrac{72 L(\nu^2 + mG^2)H^3}{  (R_\varepsilon H) (R_\varepsilon H - T_{\hat{r}})} \leq \tfrac{144 L(\nu^2 + mG^2)H^3}{  (R_\varepsilon H)^{2}}.\end{align}
Next, we write
\begin{align}\label{eqn:thm2_eq3}
\tfrac{\fm{16\gamma m^2\fy{\delta_{0}}^2 G^2 H}}{(T_R-T_{\hat{r}})}\left(\tfrac{1-\rho^{2R}}{1-\rho^2}\right) 
&\leq \tfrac{\fm{16 m^2\fy{\delta_{0}}^2 G^2H}}{ (1-\rho^2) (R_\varepsilon H)^{1/2} (R_\varepsilon H - T_{\hat{r}})} 
\nonumber\\
&\leq \tfrac{\fm{32 m^2\fy{\delta_{0}}^2 G^2H}}{ (1-\rho^2) (R_\varepsilon H)^{3/2}}.
\end{align} In a similar vein, we have
\begin{align}\label{eqn:thm2_eq4}
&\tfrac{\fm{4\fy{\delta_{0}}\sqrt{\fy{\delta_{0}}}(G^2 + \sum_{i=1}^m \|\nabla f_i(x^*)\|^2) H}}{  (T_R-T_{\hat{r}})}\left(\tfrac{1-\rho^{1.5R}}{1-\rho^{1.5}} \right)\nonumber\\
&\leq 
\tfrac{\fm{4\fy{\delta_{0}}\sqrt{\fy{\delta_{0}}}(G^2 + \sum_{i=1}^m \|\nabla f_i(x^*)\|^2) H}}{ (1-\rho^{1.5}) (R_\varepsilon H - T_{\hat{r}})} \nonumber\\
&\leq \tfrac{\fm{8\fy{\delta_{0}}\sqrt{\fy{\delta_{0}}}(G^2 + \sum_{i=1}^m \|\nabla f_i(x^*)\|^2) H}}{ (1-\rho^{1.5}) R_\varepsilon H}.
\end{align}}
From the preceding inequalities~\eqref{eqn:thm2_eq1}—\eqref{eqn:thm2_eq4}, we have
\fm{\begin{align*}
\mathbb{E}[f(\bar{x}_T^{\text{avg}})] - f^* &\leq \hat{b}\left(\tfrac{\fm{D_{T_{\hat{r}}}^2}}{(R_\varepsilon H)^{1/2}} + \tfrac{  L(\nu^2 + mG^2)H^3}{(R_\varepsilon H)^{2}} \right.\\
&\left.+ \tfrac{ \nu^2}{(R_\varepsilon H)^{1/2}} + \tfrac{\fm{  m^2\fy{\delta_{0}}^2 G^2H}}{(1-\rho^2) (R_\varepsilon H)^{3/2}}\right.\\
&\left.+ \tfrac{\fm{ \fy{\delta_{0}}\sqrt{\fy{\delta_{0}}}(G^2 + \sum_{i=1}^m \|\nabla f_i(x^*)\|^2) H}}{(1-\rho^{1.5}) R_\varepsilon H}\right)
\end{align*}}
The complexity bound on $R$ follows by enforcing that the preceding upper bound does not exceed~$\varepsilon$.

\end{document}